\documentclass{article}

    \PassOptionsToPackage{numbers, compress}{natbib}

    \usepackage[final]{neurips_2021}

\usepackage[numbers]{natbib}

\usepackage[utf8]{inputenc} 
\usepackage[T1]{fontenc}    
\usepackage{hyperref}       
\hypersetup{
    colorlinks=true,
    linkcolor=blue,
    filecolor=magenta,      
    urlcolor=blue,
    citecolor=blue,
    }
\usepackage{url}            
\usepackage{booktabs}       
\usepackage{amsfonts}       
\usepackage{nicefrac}       
\usepackage{microtype}      
\usepackage{xcolor}         
\usepackage{amsmath}
\usepackage{mathtools}
\usepackage{amssymb}
\usepackage{amsthm}
\usepackage{pifont}
\usepackage[linesnumbered,ruled]{algorithm2e} 
\usepackage{wrapfig}
\usepackage{algorithmic}  
\usepackage{enumitem}
\SetKwRepeat{Do}{do}{while}
\usepackage{bm}

\usepackage{selectp}

\newcommand{\al}{\alpha}
\newcommand{\be}{\beta}

\newcommand{\eps}{\epsilon}

\newcommand{\tnabla}{\tilde{\nabla}}
\newcommand{\ttheta}{\tilde{\theta}}

\newcommand{\bigO}{\mathcal{O}}

\newcommand{\R}{{\mathbb{R}}}

\newcommand{\T}{\mathcal{T}}
\newcommand{\M}{\mathcal{M}}
\newcommand{\MS}{\mathcal{S}}
\newcommand{\MR}{\mathcal{R}}
\newcommand{\A}{\mathcal{A}}
\newcommand{\F}{\mathcal{F}}
\newcommand{\I}{\mathcal{I}}
\newcommand{\B}{\mathcal{B}}
\newcommand{\D}{\mathcal{D}}
\newcommand{\E}{\mathbb{E}}

\newtheorem{theorem}{Theorem}  
\newtheorem{definition}{Definition}

\newtheorem{proposition}{Proposition}
\newtheorem{lemma}{Lemma}
\newtheorem{remark}{Remark}
\newtheorem{corollary}{Corollary}

\newtheorem{assumption}{Assumption}

\newcommand{\beq}{\begin{equation}}
\newcommand{\eeq}{\end{equation}}
\newcommand{\beqa}{\begin{eqnarray}}
\newcommand{\eeqa}{\end{eqnarray}}
\newcommand{\beqs}{\begin{equation*}}
\newcommand{\eeqs}{\end{equation*}}
\newcommand{\beqas}{\begin{eqnarray*}}
\newcommand{\eeqas}{\end{eqnarray*}}

\newcommand{\ouralg}{SG-MRL}

\title{On the Convergence Theory of Debiased Model-Agnostic Meta-Reinforcement Learning}

\author{%
  Alireza Fallah \\
   EECS Department \\
   Massachusetts Institute of Technology \\
  \texttt{afallah@mit.edu} \\
   \And
   Kristian Georgiev \\
   EECS Department \\
   Massachusetts Institute of Technology \\
   \texttt{krisgrg@mit.edu} \\
   \And
   Aryan Mokhtari \\
   ECE Department \\
   The University of Texas at Austin \\
   \texttt{mokhtari@austin.utexas.edu} \\
   \And
   Asuman Ozdaglar \\
   EECS Department \\
   Massachusetts Institute of Technology \\
   \texttt{asuman@mit.edu} \\
}

\begin{document}

\maketitle

\begin{abstract}
We consider Model-Agnostic Meta-Learning (MAML) methods for Reinforcement Learning (RL) problems, where the goal is to find a policy using data from several tasks represented by Markov Decision Processes (MDPs) that can be updated by one step of \textit{stochastic} policy gradient for the realized MDP. In particular, using stochastic gradients in MAML update steps is crucial for RL problems since computation of exact gradients requires access to a large number of possible trajectories. For this formulation, we propose a variant of the MAML method, named Stochastic Gradient Meta-Reinforcement Learning (SG-MRL), and study its convergence properties. We derive the iteration and sample complexity of SG-MRL to find an $\epsilon$-first-order stationary point, which, to the best of our knowledge, provides the first convergence guarantee for model-agnostic meta-reinforcement learning algorithms. We further show how our results extend to the case where more than one step of stochastic policy gradient method is used at test time. Finally, we empirically compare SG-MRL and MAML in several deep RL environments.
\end{abstract}

\section{Introduction}

Meta-learning has recently attracted much attention as a learning to learn approach that enables quick adaptation to new tasks using past experience and data. This is a particularly promising approach for Reinforcement Learning (RL) where in several applications, such as robotics, a group of agents encounter new tasks and need to learn new behaviors or policies through a few interactions with the environment building on previous experience \citep{finn17a, duan2016rl,wang2016learning,mishra2017simple,rothfuss2018promp,wang2018prefrontal,nagabandi2018learning,rakelly2019efficient,yu2019meta}. Among various forms of  Meta-learning, gradient-based Model-Agnostic Meta-Learning (MAML) formulation \citep{finn17a} is a particularly effective approach which, as its name suggests, can be applied to any learning problem that is trained with gradient-based updates. In MAML, we exploit observed tasks at training time to find an initial model that is trained in a way that rapidly adapts to a new unseen task at test time, after running a few steps of a gradient-based update with respect to the loss of the new task.

The MAML formulation can be extended to RL problems if we represent each task as a Markov Decision Process (MDP). 
{In this setting, we assume that} we are given a set of MDPs corresponding to the tasks that we observe during the training phase and assume that the new task at test time is drawn from an underlying probability distribution. The goal in Model-Agnostic Meta-Reinforcement Learning (MAMRL) is to exploit this data to come up with an initial policy that adapts to a new task (drawn from the same distribution) at test time by taking a few stochastic policy gradient steps~\citep{finn17a}. 
 
 Several algorithms have been proposed in the context of MAMRL \citep{finn17a,yu2019meta,  liu2019taming, mendonca2019guided, gupta2018meta} which demonstrate the advantage of this framework in practice. None of these methods, however, are supported by theoretical guarantees for their convergence rate or overall sample complexity. Moreover, these methods aim to solve a specific form of MAMRL that 
 {does not fully take into account the stochasticity aspect of RL problems.}
 To be more specific, the original MAMRL formulation proposed in \cite{finn17a} assumes performing one step of \textit{policy gradient} to update the initial model at test time. However, as mentioned in the experimental evaluation section in \cite{finn17a}, it is more common in practice to use \textit{stochastic} policy gradient, computed over a batch of trajectories, to update the initial model at test time. This is mainly due to the fact that computing the exact gradient of the expected reward is not computationally  tractable due to the massive number of possible state-action trajectories. As a result, the algorithm developed in \cite{finn17a} is designed for finding a proper initial policy that performs well after one step of policy gradient, while in practice it is implemented with stochastic policy gradient steps. Due to this difference between the formulation and what is used in practice, the ascent step used in MAML takes a gradient estimate which suffers from a non-diminishing  \textit{bias}. As the variance of gradient estimation is also non-diminishing, the resulting algorithm would not achieve exact first-order optimality. To be precise, in stochastic nonconvex optimization, if we use an unbiased gradient estimator, along with a small stepsize or a large batch size to control the variance, the iterates converge to a stationary point. However, if we use a biased estimator with non-vanishing bias and variance, exact convergence to a stationary point is not achievable, even if the variance is small.



\noindent \textbf{Contributions.}
The goal of this paper is to solve the modified formulation of model-agnostic meta-reinforcement learning problem in which we perform a stochastic policy gradient update at test time instead of (deterministic) policy gradient. To do so, we propose a novel stochastic gradient-based method for Meta-Reinforcement Learning~(SG-MRL), which is designed for \textit{stochastic policy gradient} steps at test time.
We show that SG-MRL implements an \textit{unbiased} estimate of its objective function gradient which allows achieving first-order optimality in non-concave settings. Moreover, we characterize the relation between batch sizes and other problem parameters and the best accuracy that SG-MRL can achieve in terms of gradient norm. We show that, for any $\epsilon > 0$, SG-MRL can find an $\eps$-first-order stationary point if the learning rate is sufficiently small or the batch of tasks is large enough. To the best of our knowledge, this is the first result on the convergence of MAMRL methods. Moreover, we show that our analysis can be extended to the case where more than one step of stochastic policy gradient is taken during test time. For simplicity, we state all the results in the body of the paper for the single-step case and include the derivations of the general multiple steps case in the appendices.
We also empirically validate the proposed SG-MRL algorithm in larger-scale environments standard in modern reinforcement learning applications, including a 2D-navigation problem, and a more challenging locomotion problem simulated with the MuJoCo library.

\noindent \textbf{Related work.} Although this paper provides the first theoretical study of MAML for RL, several recent papers have studied the complexity analysis of MAML in other contexts. In particular, the iMAML algorithm which performs an approximation of one step of proximal point method (instead of a few steps of gradient descent) in the inner loop was proposed in \cite{rajeswaran2019meta}. The authors focus on the deterministic case, and show that, assuming the inner loop loss function is sufficiently smooth, i.e., the regularized inner loop function is strongly convex, iMAML converges to a first-order stationary point. Another recent work \cite{fallah2019convergence} establishes convergence guarantees of the MAML method to first-order stationarity for non-convex settings. Also,  \cite{ji2020multi} extends the theoretical framework in \cite{fallah2019convergence} to the multiple-step case. However, the results in \cite{fallah2019convergence,ji2020multi} cannot be applied to the reinforcement learning setting. This is mainly due to the fact that \textit{the probability distribution over possible trajectories of states and actions varies with the policy parameter}, leading to a different algorithm that has an additional term which makes the analysis, such as deriving an upper bound on the smoothness parameter, more challenging. We will discuss this point in subsequent sections. 

The online meta-learning setting has also been studied in a number of recent works \citep{finn19a, balcan_ICML, khodak2019adaptive}. In particular, \cite{balcan_ICML} studies this problem for convex objective functions by casting it in the online convex optimization framework.  Also, \cite{finn19a} extends the model-agnostic setup to the online learning case by considering a competitor which adapts to new tasks, and propose the follow the meta leader method which obtains a sublinear regret for strongly convex loss functions.

 It is also worth noting that another notion of bias that has been studied in the MAMRL literature \citep{liu2019taming, foerster2018dice} differs from what we consider in our paper. More specifically, as we will show later, the derivative of the MAML objective function requires access to the second-order information, i.e., Hessian. In \cite{finn17a}, the authors suggest a first-order approximation which ignores this second-order term. This leads to a biased estimate of the derivative of the MAML objective function, and a number of recent works \citep{liu2019taming, foerster2018dice} focus on providing unbiased estimates for the second-order term. In contrast, here we focus on biased gradient estimates where the bias stems from the fact that in most real settings we do not have access to all possible trajectories and we only have access to a mini-batch of possible trajectories. In this case, even if one has access to the second-order term required in the update of MAML, the bias issue we discuss here will remain.

\section{Problem formulation}\label{sec:problem}

Let $\{\M_i\}_i$ be the set of Markov Decision Processes (MDPs) representing different tasks\footnote{To simplify the analysis, we assume the number of tasks is finite}. We assume these MDPs are drawn from a distribution $p$ (which we can only draw samples from), and also the time horizon is fixed and is equal to $\{0,1,...,H\}$ for all tasks. For the $i$-th MDP denoted by $\M_i$, which corresponds to task $i$, we denote the set of states and actions by $\MS_i$ and $\A_i$, respectively. We also assume the initial distribution over states in $\MS_i$ is given by $\mu_i(\cdot)$ and the \textit{transition kernel} is denoted by $P_i$, i.e., the probability of going from state $s \in \MS_i$ to $s' \in \MS_i$ given taking action $a \in \A_i$ is $P_i(s'|s,a)$. Finally, we assume at state $s$ and by taking action $a$, the agent receives reward $r_i(s,a)$. To summarize, an MDP $\M_i$ is defined by the tuple $(\MS_i, \A_i, \mu_i, P_i, r_i)$.
For MDP $\M_i$, the actions are chosen according to a \textit{random policy} which is a mixed strategy over the set of actions and depends on the current state, i.e., if the system is in state $s \in \MS_i$, the agent chooses action $a \in \A_i$ with probability $\pi_i(a|s)$. To search over the space of all policies, we assume these policies are parametrized with $\theta  \in \R^d$, and denote the policy corresponding to parameter $\theta$ by $\pi_i(\cdot|\cdot;\theta)$.

A realization of states and actions in this setting is called a \textit{trajectory}, i.e., a trajectory of MDP $\M_i$ can be written as $\tau = (s_0,a_0,...,s_H,a_H)$ where $a_h \in \A_i$ and $s_h \in \MS_i$ for any $0 \leq h \leq H$. Note that, given the above assumptions, the probability of this particular trajectory is given by
\begin{equation}\label{traj_prob}
q_i(\tau; \theta) := \mu_i(s_0) \prod_{h=0}^H \pi_i(a_h|s_h;\theta) \prod_{h=0}^{H-1} P_i(s_{h+1}|s_h,a_h).	
\end{equation}
Also, the total reward received over this trajectory is
$\mathcal{R}_i(\tau) := \sum_{h=0}^{H}	\gamma^h r_i(s_h,a_h)$,
where $0 \leq \gamma \leq 1$ is the \textit{discount factor}. As a result, for MDP $\M_i$, the expected reward obtained by choosing policy $\pi(\cdot|\cdot;\theta)$ is given by
\begin{equation}\label{exp_reward}
J_i(\theta) := \E_{\tau \sim q_i(\cdot;\theta)} \left [ \MR_i(\tau) \right ].	
\end{equation}
It is worth noting that the gradient $\nabla J_i(\theta)$ admits the following characterization \citep{sutton2018reinforcement, peters2008reinforcement, shen2019hessian} 
\begin{equation}\label{grad_J}
\nabla J_i(\theta) =	\E_{\tau \sim q_i(\cdot;\theta)} \left [ g_i(\tau; \theta) \right ],
\end{equation}
where $g_i(\tau;\theta)$ is defined as
\begin{equation}\label{g_tau}
g_i(\tau;\theta) :=  \sum_{h=0}^H \nabla_\theta \log 	\pi_i(a_h|s_h;\theta) \MR_i^h(\tau),
\end{equation}
if we define $\MR_i^h(\tau)$ as 
$
\MR_i^h(\tau):= \sum_{t=h}^{H}	\gamma^t r_i(s_t,a_t).	
$
In practice, evaluating the exact value of \eqref{grad_J} is not computationally tractable. Instead, one could first acquire a batch $\D^{i,\theta}$ of trajectories drawn independently from distribution $q_i(\cdot;\theta)$, and then, estimate $\nabla J_i(\theta)$ by
\begin{equation}\label{grad_J_tilde}
\tnabla J_i(\theta, \D^{i,\theta}) := \frac{1}{|\D^{i,\theta}|} \sum_{\tau \in \D^{i,\theta}} g_i(\tau; \theta).
\end{equation}	
Also, we denote the probability of choosing (with replacement) an independent batch of trajectories $\D^{i,\theta}$  by $q_i(\D^{i,\theta};\theta)$ (see Appendix \ref{remark_batch} for a remark on this).

In this setting, the goal of Model-Agnostic Meta-Reinforcement Learning problem introduced in \citep{finn17a} is to find a good initial policy that performs well in expectation when it is updated using one or a few steps of \textit{stochastic policy gradient} with respect to a new task. In particular, for the case of performing one step of stochastic policy gradient, the problem can be written as\footnote{From now on, we suppress the $\theta$ dependence of batches to simplify the notation.}
\begin{equation}\label{Meta-RL_prob_1}
\max_{\theta \in \R^d} V_1(\theta):= \E_{i \sim p} \left [ \E_{\D_{test}^{i}} \left [ J_i \left ( \theta + \alpha \tnabla J_i(\theta, \D_{test}^{i}) \right ) \right ] \right ].
\end{equation}
Note that by solving this problem we find an initial policy (Meta-policy) that in expectation performs well if we evaluate the output of our procedure after running one step of stochastic policy gradient on this initial policy for a new task.

This formulation can be extended to the setting with more than one step of stochastic policy gradient as well. To state the problem formulation in this case, let us first define $\Psi_i$ which is an operator that takes model $\theta$ and batch $\D^i$ as input and performs one step of stochastic gradient policy at point $\theta$ and with respect to function $J_i$ and batch $\D^i$, i.e., 
$
\Psi_i(\theta, \D^i) := \theta + \alpha \tnabla J_i(\theta, \D^i).	
$
Now, we extend problem \eqref{Meta-RL_prob_1} to the case where we are looking for an initial point which performs well on expectation after it is updated with $\zeta$ steps of stochastic policy gradient with respect to a new MDP drawn from distribution $p$. This problem can be written as
\begin{align}\label{Meta-RL_prob_zeta}
\max_{\theta \in \R^d}  V_\zeta(\theta):= \E_{i \sim p} \bigg [
\E_{\{\!\D_{test,t}^{i}\!\}_{t=1}^{\zeta}}\!\! \left [ J_i\! \left ( \Psi_i(\ldots(\Psi_i(\theta, \D_{test,1}^{i} )\ldots),\D_{test, \zeta}^{i}) \right ) \right ] \bigg ], 
\end{align}
where the operator $\Psi_i$ is applied $\zeta$ times inside the expectation. In this paper, we establish convergence properties of policy gradient methods for both single step and multiple steps of stochastic gradient cases, but for simplicity in the main text we focus on the single step case.

\subsection{Second-order information of the expected reward}
Due to the inner gradient in $V_1(\theta)$, i.e., the objective function of the MAML problem in \eqref{Meta-RL_prob_1}, the gradient of the function $V_1(\theta)$ requires access to the second-order information of the expected reward function $J(\theta)$. To facilitate further analysis, in this subsection we formally present a characterization of expected reward Hessian and its unbiased estimate over a batch of trajectories. In particular, the expected reward Hessian $\nabla^2 J_i(\theta)$ is given by (see \cite{shen2019hessian} for more details)
\begin{align}\label{Hessian_J}
&\!\!\nabla^2 J_i(\theta)  = \E_{\tau \sim q_i(\cdot;\theta)} \left [ u_i(\tau; \theta) \right ],
\quad u_i(\tau; \theta)  \!:=\! \nabla_\theta \nu_i(\tau;\theta)  \nabla_\theta \log q_i(\tau;\theta)^\top\!\! +\! \nabla_\theta^2 \nu_i(\tau;\theta)
\end{align} 
where $\nu_i(\tau;\theta)$ is given by $\nu_i(\tau;\theta) :=  \sum_{h=0}^H \log 	\pi_i(a_h|s_h;\theta) \MR_i^h(\tau)$.

Recall that the reward function is defined as $\MR_i^h(\tau):= \sum_{t=h}^{H}	\gamma^t r_i(s_t,a_t)$. It is worth noting that based on the expression in \eqref{g_tau} we can write $g_i(\tau;\theta) = \nabla_\theta \nu_i(\tau;\theta)$. 
 
 Similar to policy gradient, policy Hessian can be estimated over a batch of trajectories $\D^{i}$ independently drawn with respect to $q_i(;\theta)$. Specifically, for a given dataset $\D^i$, we can define $\tnabla^2 J_i(\theta, \D^i)$ \begin{equation}\label{Hessian_J_tilde}
\tnabla^2 J_i(\theta, \D^i) := 	\frac{1}{|\D^i|} \sum_{\tau \in \D^{i}} u_i(\tau; \theta)
\end{equation}
as an unbiased estimator of the Hessian $\nabla^2 J_i(\theta)$. We will use the expressions for the Hessian $\nabla^2 J_i(\theta)$ in \eqref{Hessian_J} and the Hessian approximation $\tnabla^2 J_i(\theta, \D^i)$ in \eqref{Hessian_J_tilde} to introduce our proposed method for solving the Meta-RL problem in \eqref{Meta-RL_prob_1} and its generalized version in \eqref{Meta-RL_prob_zeta}.

\section{Model-agnostic meta reinforcement learning}\label{sec:MAML}
In this section, we first propose a method to solve the stochastic gradient-based MAML Reinforcement Learning problem introduced in \eqref{Meta-RL_prob_1}. Then, we discuss how to extend the proposed method to the setting that we solve a multi-step MAML problem as introduced in \eqref{Meta-RL_prob_zeta}. We close the section by discussing the differences between our proposed method and the Meta-RL method proposed in \cite{finn17a} and clarify why these two methods are solving two different problems.



\subsection{MAML for stochastic meta-RL}

Our goal in this section is to propose an efficient method for solving the stochastic Meta-RL problem in \eqref{Meta-RL_prob_1}. 
To do so, we propose a stochastic gradient MAML method for Meta-Reinforcement Learning (\ouralg) that aims at solving problem \eqref{Meta-RL_prob_1} by following the update of stochastic gradient descent for the objective function $V_1(\theta)$. To achieve this goal one need to find an unbiased estimator of the gradient $\nabla V_1(\theta)$ which in some MAML settings is not trivial (for more details see Section 4.1 in \citep{fallah2019convergence}), but we show that for problem \eqref{Meta-RL_prob_1} an unbiased estimate of $\nabla V_1(\theta)$ can be efficiently computed. 

Let us start by pointing out that the gradient of the function $V_1(\theta)$ defined in \eqref{Meta-RL_prob_1} is given by
\begin{align}\label{grad_v_def}
&  \nabla V_1(\theta) = 	\nabla_\theta \left [ \E_{i}\ \E_{\D_{test}^{i}} \left [ J_i \left ( \theta + \alpha \tnabla J_i(\theta, \D_{test}^{i}) \right ) \right ] \right ] =  \E_i \E_{\D_{test}^{i}}\!  \bigg [ \!(I\!+\!\alpha \tnabla^2 \!J_i(\theta, \D_{test}^i)) \nonumber
\\
&  \times \nabla J_i( \theta \!+\! \alpha \tnabla J_i(\theta, \D_{test}^{i}) ) 
 + J_i  (  \theta \!+\! \alpha \tnabla J_i(\theta, \D_{test}^{i})  ) \!\!\sum_{\tau \in \D_{test}^{i}} \!\!\nabla_\theta \log 	\pi_i(\tau;\theta) \bigg]\! 
\end{align}
with the convention that for $\tau=(s_0,a_0,...,s_H,a_H)$ we define $\pi_i(\tau;\theta)$ as
\begin{equation}\label{pi_tau}
\pi_i(\tau;\theta) := \prod_{h=0}^H \pi_i(a_h|s_h;\theta).	
\end{equation}
Recall that the expected reward function $J_i(\theta)$ and its gradient $\nabla J_i(\theta)$ are defined in \eqref{exp_reward} and \eqref{grad_J}, respectively, and $\tnabla J_i(\theta, \D_{test}^{i}) $ and $\tnabla^2 J_i(\theta, \D_{test}^{i}) $ are the stochastic estimates of the gradient and Hessian corresponding to $J_i(\theta)$ that are formally defined in \eqref{grad_J_tilde} and \eqref{Hessian_J_tilde}, respectively. 

Note that the first term in the definition of  $\nabla V_1(\theta)$ in \eqref{grad_v_def}, i.e., $\!(I\!+\!\alpha \tnabla^2 \!J_i(\theta, \D_{test}^i)) \nabla J_i( \theta \!+\! \alpha \tnabla J_i(\theta, \D_{test}^{i}) )$, is the term that gives the gradient of an MAML problem (see, e.g., \citep{finn19a}), while the second term, i.e., $J_i  (  \theta + \alpha \tnabla J_i(\theta, \D_{test}^{i})  ) \sum_{\tau \in \D_{test}^{i}} \!\!\nabla_\theta \log 	\pi_i(\tau;\theta)$, is specific to the RL setting since the probability distribution $p_i$ itself depends on the parameter $\theta$. For more details regarding the derivation $\nabla V_\zeta(\theta)$ for any $\zeta \geq 1$, we refer the reader to Appendix~\ref{der_V_zeta}. 

We solve the optimization problem in \eqref{Meta-RL_prob_1} by using gradient ascent step to update the parameter $\theta$, i.e., following the update $\theta_{k+1} = \theta_k + \beta \nabla {V}_1(\theta_k)$ at iteration $k$. However, computing the gradient $\nabla {V}_1(\theta_k)$ may not be tractable in many cases due to the large number of tasks and the size of the action and state spaces. In our proposed \ouralg\  method we therefore replace the gradient $\nabla {V}_1(\theta_k)$ with its estimate computed as follows: At iteration $k+1$, we first choose a subset  $\mathcal{B}_k$  of the tasks (MDPs), where each task is drawn independently from the probability distribution $p$. The \ouralg\ outlined in Algorithm \ref{Algorithm_SG_MRL} is implemented at two levels: (i) inner loop and (ii) outer loop. In the inner loop, for each task $\T_i$ with $i \in \mathcal{B}_k$, we draw a batch of trajectories $\D_{in}^i$ according to $q_i(\cdot;\theta_k)$ to compute the stochastic gradient $\tilde{\nabla}J_i(\theta_k, \D_{in}^i)$ as defined in Section \ref{sec:problem}. This estimate is then used to compute a model $\theta_{k+1}^i$ corresponding to task $\T_i$ by a single iteration of stochastic policy gradient,
\begin{equation}\label{SG_MRL_inner_update}
\theta_{k+1}^i = \theta_k + \alpha \tilde{\nabla}J_i(\theta_k,\D_{in}^i).
\end{equation}
For simplicity, we assume that the size of $\mathcal{B}_k$ is equal to $B$ for all $k$, and the size of dataset $\D_{in}^i$ is fixed for all tasks and at each iteration, and we denote it by $D_{in}$.

In the outer loop, we compute the next iterate $\theta_{k+1}$ using the iterates  $\{\theta_{k+1}^i\}_{i \in \B_k}$ that are computed in the inner loop. In particular, we follow the update
$\theta_{k+1} = \theta_k +  \beta \tnabla V_1(\theta_k),$
where 
\begin{align}\label{unbiased_V_1}
\tnabla V_1(\theta_k) :=
 \frac{1}{B} & \sum_{i \in \B_k} \bigg[ (I + \al \tilde{\nabla}^2 J_i(\theta_{k},\D_{in}^i)) \tnabla \!J_i (\theta_k + \alpha \tilde{\nabla}\!J_i(\theta_k,\D_{in}^i), \D_{o}^i ) \\ 
&  \qquad \qquad + \tilde{J}_i \left ( \theta_k + \alpha \tilde{\nabla}J_i(\theta_k,\D_{in}^i), \D_{o}^i \right ) \sum_{\tau \in \D_{in}^{i}} \nabla_\theta \log 	\pi_i(\tau;\theta_k) \bigg] \nonumber
\end{align}
in which
$\tilde{\nabla}^2 J_i(\theta_{k},\D_{in}^i)$ is policy Hessian estimate defined in~\eqref{Hessian_J_tilde} and for each task $\T_i$, the dataset $\D_{o}^i$ is a new batch of trajectories that are drawn based on the probability distribution $q_i(\cdot;\theta_{k+1}^{i})$; Again, for simplicity, we assume that the size of dataset $\D_{o}^i$ is fixed for all tasks and at each iteration denoted by $D_{o}$. \ouralg\ is summarized in Algorithm \ref{Algorithm_SG_MRL}.

\begin{algorithm}[tb]
\caption{Proposed \ouralg\ method for Meta-RL}
\label{Algorithm_SG_MRL} 
\begin{algorithmic}
    \STATE {\bfseries Input:}Initial iterate $\theta_0$
	\REPEAT
    \STATE Draw a batch of i.i.d. tasks  $\B_k\! \subseteq\! \I$ with size $B = |\B_k|$;
    \FOR{all $\T_i$ with $i \in \B_k$}
    \STATE Sample a batch of trajectories  $\D_{in}^{i}$ w.r.t. $q_i(\cdot;\theta_k)$;
    \STATE Set $\theta_{k+1}^{i} = \theta_{k} + \al \tilde{\nabla}J_i(\theta_{k},\D_{in}^{i})$;
    \ENDFOR
         \STATE Sample a batch of trajectories $\D_{o}^i$ w.r.t.  $q_i(\cdot;\theta_{k+1}^{i})$;
    \STATE Set $\theta_{k+1} = \theta_k$
    \vspace{-4mm}
    \STATE $\displaystyle{ + \frac{\be}{B} \sum_{i \in \B_k} \bigg( \left(I + \al \tilde{\nabla}^2 J_i(\theta_{k},\D_{in}^i)\right) \tnabla J_i \left(\theta_{k+1}^{i}, \D_{o}^i \right)
    + \overbrace{\tilde{J}_i \left ( \theta^{i}_{k+1}, \D_{o}^i \right ) \sum_{\tau \in \D_{in}^{i}} \nabla_\theta \log 	\pi_i(\tau;\theta_k)}^{\text{Additional term in SG-MRL}} \bigg)}$  
    \STATE $k \gets k + 1$
    \UNTIL{not done}
    \end{algorithmic}
    \end{algorithm}

It can be verified that if all the gradients and Hessians in SG-MRL update were exact, then the outcome of the update of \ouralg\ would be equivalent to the outcome of gradient ascent update for the function $V_1$, i.e., $\theta_{k+1} = \theta_k + \beta \nabla {V}_1(\theta_k)$. 
Note that by computing the expected value of $\tnabla V_1(\theta_k)$ first with respect to the random set $\D_{o}^i$, then with respect to $\D_{in}$, and finally with respect to $\B_k$, we obtain that $\E[\tnabla V_1(\theta_k)]=\nabla V_1(\theta_k)$. Therefore, the stochastic gradient $\tnabla V_1(\theta_k)$ is an unbiased estimator of the gradient $\nabla V_1(\theta_k)$.

The \ouralg\ method can also be extended and used for solving the multi-step MAML problem defined in \eqref{Meta-RL_prob_zeta}. To do so, at each iteration, we first perform $\zeta$ steps of policy stochastic gradient in the inner loop, and then take one step of stochastic gradient ascent with respect to an unbiased estimator of $\nabla V_\zeta(\theta)$. More details on the implementation of \ouralg\  for that case is provided in Appendix~\ref{der_V_zeta}.


\subsection{Comparing \ouralg\ with other model-agnostic meta-RL methods}
In this section, we discuss the difference between our \ouralg\ method and recent Meta-RL methods. In particular, we focus on the MAML method in \cite{finn17a} for solving RL problems. Before discussing the differences between these two methods, let us first recap the update of the MAML method in \cite{finn17a}.

The main formulation proposed in \cite{finn17a} which was followed in other works such as 
\cite{liu2019taming} is slightly different from the one in this paper as they assume the agent has access to the \textit{exact gradient} of the new task, and hence, they consider the following MAML problem
\begin{equation}\label{Meta-RL_prob_1_det}
\max_{\theta \in \R^d} \hat{V}_1(\theta):= \E_{i \sim p} \left [ J_i \left ( \theta + \alpha \nabla J_i(\theta) \right ) \right ].
\end{equation}
As mentioned, the main difference between \eqref{Meta-RL_prob_1} and \eqref{Meta-RL_prob_1_det} is that the former tries to find a good initial policy that leads to a good solution after running one step of \textit{stochastic gradient ascent}, while the latter finds an initial policy that produces a good policy after running one step of \textit{gradient ascent}. 

\begin{remark}
Problems in \eqref{Meta-RL_prob_1} and \eqref{Meta-RL_prob_1_det} are both valid formulations for Meta-RL. In practice, however, it is often computationally intractable to evaluate the exact gradient of the expected reward and we often have only access to its stochastic gradient. Hence, it might be more practical to solve \eqref{Meta-RL_prob_1} instead of \eqref{Meta-RL_prob_1_det} as it finds an initial policy that performs well after running one step of stochastic gradient, unlike \eqref{Meta-RL_prob_1_det}  that finds a policy that performs well after running one step of gradient update.
\end{remark}

In a nutshell, the MAML method proposed in \cite{finn17a} tries to solve the problem in  \eqref{Meta-RL_prob_1_det} by following the update of stochastic gradient ascent for the objective function $\hat{V}_1(\theta)$. To be more precise, note that the gradient of the loss function $\hat{V}_1(\theta)$ defined in \eqref{Meta-RL_prob_1_det} can be expressed as
\begin{equation}\label{grad_V_1_hat}
\nabla \hat{V}_1(\theta)  = 	\nabla_\theta \E_{i \sim p} \left [ J_i \left ( \theta + \alpha \nabla J_i(\theta) \right ) \right ]   = \E_{i \sim p} \left [ \left (I + \alpha \nabla^2 J_i(\theta) \right ) \nabla J_i \left ( \theta + \alpha \nabla J_i(\theta) \right ) \right ]. 
\end{equation}
\sloppy
Note that the expression for the gradient of $\hat{V}_1(\theta)$ in \eqref{grad_V_1_hat} is different from the expression for the gradient of $ {V}_1(\theta)$ in \eqref{grad_v_def}. In particular, the extra term $J_i  (  \theta + \alpha \tnabla J_i(\theta, \D_{test}^{i})  ) \sum_{\tau \in \D_{test}^{i}} \!\!\nabla_\theta \log 	\pi_i(\tau;\theta)$ that appears in $ \eqref{grad_V_1_hat}$ is caused by the fact that we use stochastic gradients in the definition of the function ${V}_1(\theta)$, while exact gradients are used in the definition of $ \hat{V}_1(\theta)$. 

Considering the expression for the gradient of $ \hat{V}_1(\theta)$ in \eqref{grad_V_1_hat}, a natural approach to approximate $\nabla \hat{V}_1(\theta) $ is to replace the gradients and Hessians corresponding to the expected reward $J_i(\theta)$ by their stochastic approximations. In other words, one can use the approximation $\tnabla \hat{V}_1(\theta_k)$ which is defined as the average over $(I + \al \tilde{\nabla}^2 J_i(\theta_{k},\D_{in}^i)) \tnabla J_i (\theta_k+\al \tilde{\nabla} J_i(\theta_k,\D_{in}^i), \D_{o}^i )$ for all $i \in \B_k$, i.e., 
\begin{align}\label{stoch_grad_approx_2}
\tnabla \hat{V}_1(\theta_k):=\frac{1}{B} \sum_{i\in \B_k} \left(I + \al \tilde{\nabla}^2 J_i(\theta_{k},\D_{in}^i)\right) \tnabla J_i \left(\theta_{k}^i, \D_{o}^i \right)
\end{align}
where $\theta_{k}^i:= \theta_k+\al \tilde{\nabla} J_i(\theta_k,\D_{in}^i)$. Here the procedure for computing the sample sets $\D_{in}^i$ and $\D_{o}^i $ is the same as the one in \ouralg. Once $\tnabla \hat{V}_1(\theta_k)$ is computed the new variable $\theta_{k+1}$ can be computed by following the update of stochastic gradient ascent, i.e.,  
$
\theta_{k+1}=\theta_{k}+\beta\  \tnabla \hat{V}_1(\theta_k)
$.
The description of the Meta-RL method in \cite{finn17a} and its implementation at two levels (inner and outer) is similar to the one in Algorithm \ref{Algorithm_SG_MRL}, except the highlighted additional term which is not included in MAML update.

\sloppy
Note that the gradient estimate $\tnabla \hat{V}_1(\theta_k)$ in \eqref{stoch_grad_approx_2} is a \textit{biased} estimate of the exact gradient $\nabla \hat{V}_1(\theta_k) $ defined in \eqref{grad_V_1_hat}. This is due to the fact that $\tnabla J_i (\theta_k+\al \tilde{\nabla} J_i(\theta_k,\D_{in}^i), \D_{o}^i )$ is a biased estimate of $ \nabla J_i ( \theta + \alpha \nabla J_i(\theta)  )$ because of the term $\tilde{\nabla} J_i(\theta_k,\D_{in}^i)$ inside it. In other words, MAML method proposed by \cite{finn17a} uses a biased estimate of the gradient in this case. 
Note that, in general optimization analyses, when we have access to biased gradient estimators, even with diminishing or small stepsize, we might only converge to a neighborhood of the optimal solution, where the radius of our convergence depends on the bias. To resolve this issue, one needs to control the bias in the gradient directions and lower the bias as time progresses using some debiasing techniques. For instance, the work in \cite{hu2020biased} studies this problem in detail for debiasing MAML in the supervised learning setting. 

On the other hand, our proposed \ouralg\ method  does not suffer from this issue since computing an unbiased estimator of the gradient for the objective function considered in \eqref{Meta-RL_prob_1} is relatively simple. In fact, in the following section, we show that  \ouralg\  is provably convergent and characterize its complexity to find an approximate first-order stationary point of \eqref{Meta-RL_prob_1} and its generalized version defined in \eqref{Meta-RL_prob_zeta}.

\section{Theoretical results}\label{sec:Theory}

In this section, we study the convergence properties of the proposed \ouralg\ method and characterize its overall complexity for finding a policy that satisfies the first-order optimality condition for the objective function $V_\zeta (\theta)$ defined in \eqref{Meta-RL_prob_zeta}. To do so, we first formally define the first-order optimality condition that we aim to achieve.

\begin{definition}
A random vector $\theta_\eps\in \R^d$ is called an $\epsilon$-approximate first-order stationary point (FOSP) for problem \eqref{Meta-RL_prob_zeta} if it satisfies
$
\E[ \Vert \nabla V_\zeta(\theta_\eps) \Vert] \leq \eps.    
$
\end{definition} 
We next state the main assumptions that we use to derive our results. 
\begin{assumption}\label{assump_1}
The reward functions $r_i$ are nonnegative and uniformly bounded, i.e., there exists a constant $R$ such that for any task $i$, state $s \in \MS_i$, and action $a \in \A_i$, we have
$
0 \leq r_i(a|s) \leq R. 	
$
\end{assumption}
\begin{assumption}\label{assump_2}
There exist constants $G$ and $L$ such that for any $i$ and for any state $s \in \MS_i$, action $a \in \A_i$, and parameter $\theta \in \R^d$, we have
$
\| \nabla_\theta \log 	\pi_i(a|s;\theta) \| \leq G$ and $  \| \nabla_\theta^2 \log 	\pi_i(a|s;\theta) \| \leq L$.
\end{assumption}
Both assumptions are customary in the policy gradient literature and have been used in other papers to obtain convergence guarantees for policy gradient methods \citep{pmlr-v80-papini18a,shen2019hessian, agarwal2019optimality}. 
\begin{assumption}\label{assump_3}
There exists a constant $\rho$ such that for any $i$ and for any state $s \in \MS_i$, action $a \in \A_i$, and parameters $\theta_1, \theta_2 \in \R^d$, we have
$
\| \nabla_\theta^2 \log 	\pi_i(a|s;\theta_1) - \nabla_\theta^2 \log 	\pi_i(a|s;\theta_2) \| \leq \rho \| \theta_1 - \theta_2 \|
$.
\end{assumption}
This assumption is also customary in the analysis of MAML-type algorithms \cite{fallah2019convergence, finn19a}. In particular, in Appendix~\ref{sec:softmax_example} we provide more insight into the conditions in Assumptions~\ref{assump_2} and \ref{assump_3} by focusing on the special case of \textit{softmax policy parametrization}.
\subsection{Convergence of \ouralg}
Next, we study the convergence of our proposed \ouralg\ for solving the Model-Agnostic Meta-Reinforcement Learning problem in \eqref{Meta-RL_prob_zeta}. To do so, we show two important intermediate results. First, we show that the function $V_\zeta(\theta)$ is smooth. Second, we show the unbiased estimator of the gradient $\nabla V_\zeta(\theta)$ denoted by $ \tnabla V_\zeta(\theta_k) $ has a bounded norm. Building on these two results, we will derive the convergence of \ouralg. To prove these two intermediate results, we first state the following lemma on the Lipschitz property of the expected reward function $J_i$ and its first and second derivatives for any MDP $\M_i$. This lemma not only plays a key role in our analysis, but also can be of independent interest in general for analyzing meta-reinforcement learning algorithms.
\begin{lemma}\label{lemma_all_in_all}
Recall the definitions of $g_i(\tau;\theta)$ in \eqref{g_tau}  and $u_i(\tau;\theta)$ in \eqref{Hessian_J} for trajectory $\tau \in (\MS_i \times \A_i)^{H+1}$ and policy parameter $\theta \in \R^d$. If Assumptions~\ref{assump_1}-\ref{assump_3} hold, then for any MDP $\M_i$ we have:

\quad \textbf{i)} 
For any $\tau$ and $\theta$, we have
$
\|g_i(\tau;\theta) \| \leq 	\eta_G:= \frac{GR}{(1-\gamma)^2}.
$
As a consequence, $\|\nabla J_i(\theta)\|, \|\tnabla J_i(\theta, \D^i)\| \leq \eta_G$ for any $\theta$ and any batch of trajectories $\D^i$. Further, this implies that $J_i(.)$ is smooth with parameter $\eta_G$.

\quad \textbf{ii)}  For any $\tau$ and $\theta$, we have
$
\|u_i(\tau;\theta) \| \leq 	\eta_H:= \frac{((H+1)G^2+L)R}{(1-\gamma)^2}.
$
As a consequence, $\|\nabla^2 J_i(\theta)\|, \|\tnabla^2 J_i(\theta, \D^i)\| \leq \eta_H$ for any $\theta$ and any batch of trajectories $\D^i$. Further, this implies that $\nabla J_i(.)$ is smooth with parameter $\eta_H$.

\quad \textbf{iii)}  For any  batch of trajectories $\D^i$, $\tnabla^2 J_i (\theta, \D^i)$ is smooth with parameter $\eta_\rho := \frac{ (2(H+1)GL + \rho) R}{(1-\gamma)^2}$.
\end{lemma}
By exploiting the results in Lemma~\ref{lemma_all_in_all}, we can prove the promised results on the Lipschitz property of $\nabla V_\zeta(\theta)$ as well as boundedness of its unbiased estimator $\tnabla V_\zeta(\theta)$. In the following proposition, due to space limitation and for the the ease of notation we only state the result for the case that $\zeta=1$; however, the general version of these results along with their proofs are available in Appendix \ref{app_smoothness}.
\begin{proposition}\label{thm_der_V_1_parameters}
Consider the objective function $V_1$ defined in \eqref{Meta-RL_prob_1} for the case that $\alpha \in (0, {1}/{\eta_H}]$ where $\eta_H$ is given in Lemma \ref{lemma_all_in_all}. Suppose that the conditions in Assumptions~\ref{assump_1}-\ref{assump_3} are satisfied. Then, 

\quad \textbf{i)}
$V_1(\theta)$ is smooth with parameter 
\begin{align}\label{smoothness_grad_V_1}
L_V&:= \alpha \eta_\rho \eta_G + 4 \eta_H
 + 8 R D_{in}(H+1) (L + D_{in}G^2(H+1))
\end{align}
where $\eta_G$ and $\eta_\rho$ are defined in Lemma \ref{lemma_all_in_all}. 

\quad \textbf{ii)}
For any choices of $\B_k$, $\{\D_o^i\}_i$ and $\{\D_{in}^i\}_{i}$, the norm of stochastic gradient $\tnabla V_1 (\theta_k)$ defined in \eqref{unbiased_V_1} at iteration $k$ is bounded above by
$
\!\!\| \tnabla V_1 (\theta_k) \| \leq G_V\!:=\! 2 GR\! \left[(1\!-\!\gamma)^{-2}+D_{in}(H+1) \right]\!.
$
\end{proposition}
The smoothness parameter for the RL problem has been previously characterized (as an example see \cite{shen2019hessian}), but, to the best of our knowledge, this is the first result on the smoothness parameter of the meta-RL function. Proving Proposition \ref{thm_der_V_1_parameters} is the main challenge in our analysis, since it establishes that our formulation satisfies the relevant assumptions needed for our main result in the next theorem.

Now, we present our main result on the convergence of \ouralg\ to a first-order stationary point for the Meta-reinforcement learning problem in defined \eqref{Meta-RL_prob_zeta}. We state our main result for the special case of $\zeta=1$, but the general statement of the theorem along with its proof can be found in Appendix \ref{proof_Thm_Main_Result}. 

\begin{theorem}\label{Thm_Main_Result_1} 
Consider $V_1$ defined in \eqref{Meta-RL_prob_1} for the case that $\alpha \in (0, {1}/{\eta_H}]$ where $\eta_H$ is defined in Lemma~\ref{lemma_all_in_all}. Suppose  Assumptions~\ref{assump_1}-\ref{assump_3} are satisfied, and recall the definitions of $L_V$ and $G_V$ from Proposition~\ref{thm_der_V_1_parameters}. 
Consider running SG-MRL (Algorithm \ref{Algorithm_SG_MRL}) with $\beta \in (0, 1/L_V]$.
Then, for any $1>\eps >0$, SG-MRL finds a solution $\theta_\eps$ such that
$ \E[ \| \nabla V_1(\theta_\eps) \|^2] \leq \frac{2G_V^2 L_V \beta}{BD_o} + \eps^2$,
after running for at most
$\bigO(1) \frac{R}{\beta} \min \left \{ \frac{1}{\eps^2},
\frac{B D_o}{G_V^2 L_V \beta} \right \}$ iterations.
\end{theorem}

Next we characterize the complexity of SG-MRL for finding an $\eps$-first-order stationary point solution.

\begin{corollary}\label{main_cor}
Suppose the hypotheses of Theorem \ref{Thm_Main_Result_1} hold. Then, for any $\eps >0$, SG-MRL achieves $\eps$-first-order stationarity by setting: (i)
 $BD_o \geq 8 G_V^2/\eps^2$ and $\beta =1/L_V$ requiring $\bigO ({\eps^{-2}} )$ iterations and computing  $\bigO ({\eps^{-2}} )$ stochastic gradients per iteration; or
(ii) $\beta= \bigO ({\eps^{-2}} )$ {and $BD_o = \bigO(1)$} which requires $\bigO ({\eps^{-4}} )$ iterations and $\bigO (1)$ stochastic gradient evaluations per iteration.
\end{corollary}

The conditions in Corollary \ref{main_cor} identify two settings under which SG-MRL finds an $\eps-$FOSP after a finite number of iterations, abd both settings overall require $\bigO ({\eps^{-4}} )$ stochastic gradient evaluations.

\begin{remark}
While we mainly focused on the case $\zeta=1$, we provide the general statement of the results for any $\zeta$ in the Appendix. Note that the downside of increasing $\zeta$ is that the smoothness parameter grows exponentially with respect to $\zeta$ (see Theorem \ref{thm_der_V_zeta_parameters}), which means that we need to take a smaller learning rate that leads to a slower convergence rate. However, on the positive side, by increasing $\zeta$ we train a model that better adapts to a new task. 
\end{remark}

\section{Numerical experiments}\label{sec:exp}

In this section, we empirically validate the proposed SG-MRL algorithm in larger-scale environments standard in modern reinforcement learning applications. The code is available online\footnote{The code is available at \url{https://github.com/kristian-georgiev/SGMRL}.}.

We conduct two experiments: a $2$D-navigation problem, and a more challenging locomotion problem simulated with the MuJoCo library~\citep{mujoco}.
For both experiments, we use a neural network policy with a standard feed-forward neural network and optimize it with vanilla policy gradient~\citep{williams1992simple}. Further implementation details are outlined in Appendix~\ref{appendix_experiments}. 

All experiments were conducted in MIT's Supercloud~\citep{reuther2018interactive}.
Similar to FO-MAML proposed in \cite{finn17a}, we use first order implementation of SG-MRL. It is also worth noting that SG-MRL is straightforward to implement as a modification to MAML and requires no additional hyperparameter tuning. Also, SG-MRL does not reduce the scalability of MAML. In particular, across experiments, we benchmarked the clock time of SG-MRL against MAML and SG-MRL is consistently at most $1.05$ times slower over the course of training. 
Next, we demonstrate the practicality of SG-MRL in modern deep reinforcement learning problems.

\begin{wrapfigure}{r}{0.4\textwidth}
\centering
 \includegraphics[scale=0.14]{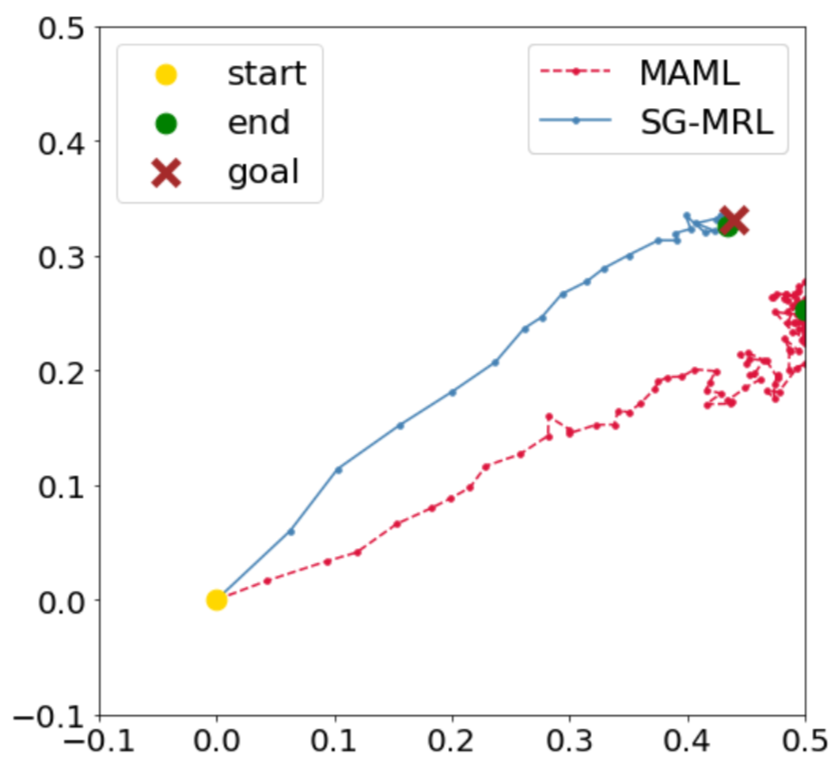}
\caption{Trajectories generated by policies trained with SG-MRL and MAML for the $2$D-navigation problem.}
\label{fig:2d_nav}
\end{wrapfigure}
\textbf{2D-navigation.} 
We consider the problem of a point-mass agent navigating from the origin to a random goal location within a unit-size square centered at the origin ($[-0.5, 0.5]\times [-0.5,0.5]$). We consider the negative squared distance to the goal location as a reward. Observations consist of the position of the agent within the unit-size square. The action space comprises of all velocities with components clipped in the interval $[-0.1,0.1]$. An example of a trajectory is illustrated in Figure~\ref{fig:2d_nav}. In Table~\ref{table:2dnav}, we compare the performance of SG-MRL against MAML \citep{finn17a} and E-MAML~\citep{emaml_stadie2018some}. We make a comparison with E-MAML since it has a similar spirit to our proposed SG-MRL method, but unlike the proposed algorithm, E-MAML is derived from heuristic arguments.

\textbf{Locomotion: MuJoCo environments.} In addition to the $2$D-navigation example, we provide a benchmark on a more challenging set of tasks - MuJoCo's locomotion environments. We benchmark our algorithm against MAML on three different tasks and report the results in Table~\ref{tab:experiments}. The tasks involve learning to move in a goal direction (forward/backward), or reach a target velocity. We describe each task in more detail in Appendix~\ref{appendix_experiments}.




\begin{table}[t]
    \centering
    \caption{Mean meta-test reward (negative square distance to goal location) of SG-MRL, MAML, and E-MAML after $1$ adaptation step.}
    \vspace{-1mm}
    \begin{tabular}{c c}
        \hline
        Algorithm & Meta-Test Reward \\
        \hline
\textbf{SG-MRL}  & $\pmb{-16.901 \pm 0.699}$\\
       MAML  & $-17.767 \pm 0.106$ \\ 
       E-MAML  & $-17.803 \pm 0.115$ \\ 
        \hline
    \end{tabular}
    \label{table:2dnav}
    \vspace{-2mm}
\end{table}



\begin{table}[t!]
    \centering
      \caption{The mean meta-test reward for SG-MRL and MAML on additional environments when trained and adapted with $1$, $2$, and $3$ inner updates over $4$ random seeds.}
      \vspace{-1mm}
    \begin{tabular}{lcc}
    \hline
      environment & SG-MRL reward & MAML reward  \\
    \hline
      Half-Cheetah Random Direction, 1 step  & $\bm{580.143 \pm 38.22}$ & $465.624 \pm 54.07$ \\
      Half-Cheetah Random Direction, 2 step  & $\bm{580.203 \pm 33.63}$ & $441.247 \pm 58.34$ \\
      Half-Cheetah Random Direction, 3 step  & $\bm{504.747 \pm 45.07}$ & $477.086 \pm 64.71$ \\
      Half-Cheetah Random Velocity, 1 step  & $\bm{-91.73 \pm 0.34}$ & $-92.92 \pm 0.70$ \\
      Half-Cheetah Random Velocity, 2 step  & $\bm{-52.64\pm 6.86}$ & $-56.71\pm 6.73$ \\
      Half-Cheetah Random Velocity, 3 step  & $-33.39\pm 0.67$ & $\bm{-32.48\pm 0.50}$ \\
      Swimmer Random Velocity, 1 step  & $\bm{118.77\pm 9.99}$ & $104.53\pm 24.18$ \\
      Swimmer Random Velocity, 2 step  & $\bm{134.57\pm 1.67}$ & $108.47\pm 23.36$ \\
      Swimmer Random Velocity, 3 step  & $\bm{110.91\pm 12.56}$ & $90.60\pm 14.99$ \\
    \hline
    \end{tabular}
    \label{tab:experiments}
\end{table}

\section{Conclusion and future work}\label{sec:conclusion}

We studied MAML for RL problems, considering performing a few steps of stochastic policy gradient at test time. Given this formulation, we introduced SG-MRL, and discussed how it differs from the original MAML algorithm in \cite{finn17a}.  Further, we characterized the convergence of SG-MRL method in terms of gradient norm and under a set of assumptions on the policy and reward functions. Our results show that, for any $\epsilon$, SG-MRL achieves $\epsilon$-first-order stationarity, given that either the learning rate is small enough or the multiplication of task and outer loop batch sizes is sufficiently large.

A shortcoming of our analysis is the requirement on the boundedness of  gradient norm (Assumption~\ref{assump_2}). A natural extension of our work would be extending the theoretical results to the setting that gradient norm is possibly unbounded. Moreover, our results are limited to achieving first-order optimality, while one can exploit techniques for escaping from saddle points to obtain second-order stationarity.

\section{Acknowledgment}
Alireza Fallah acknowledges support from the Apple Scholars in AI/ML PhD fellowship and the MathWorks Engineering Fellowship.
This research is sponsored by the United States Air Force Research Laboratory and the United States Air Force Artificial Intelligence Accelerator and was accomplished under Cooperative Agreement Number FA8750-19-2-1000. The views and conclusions contained in this document are those of the authors and should not be interpreted as representing the official policies, either expressed or implied, of the United States Air Force or the U.S. Government. The U.S. Government is authorized to reproduce and distribute reprints for Government purposes notwithstanding any copyright notation herein.  
This research of Aryan Mokhtari is supported in part by NSF Grant 2007668, ARO Grant W911NF2110226, the Machine Learning Laboratory at UT Austin, and the NSF AI Institute for Foundations of Machine Learning.

\bibliographystyle{ieeetr}
\bibliography{main}

\newpage 
\appendix

\section{Intermediate Results}
\subsection{A Remark on the Batch of Trajectories}\label{remark_batch}
Recall that $q_i(\D^{i};\theta)$ denotes the probability of independently drawing batch $\D^{i}$ of trajectories with respect to $i$-th MDP and at policy parameter $\theta$. Also, as we stated in Section \ref{sec:problem}, we assume the batch of trajectories are sampled with replacement. Note that, in this case
\begin{equation}\label{prob_batch_traj}
q_i(\D^{i};\theta) = \prod_{\tau \in 	\D^{i,\theta}} q_i(\tau; \theta).
\end{equation}
However, for the case that the batch of trajectories that we draw is not ordered, we have
\begin{equation}\label{prob_batch_traj_nonordered}
q_i(\D^{i};\theta) = C_{\D^{i}} \prod_{\tau \in 	\D^{i,\theta}} q_i(\tau; \theta).
\end{equation}
with
\begin{equation*}
C_{\D^{i}} = {|\D^i|!}/{\prod_{\tau \in (\MS_i \times \A_i)^{H+1}} C_\tau !}	
\end{equation*}
where $C_\tau$ is the number of times that the particular trajectory $\tau$ is appeared in $\D^i$. Throughout the proofs, we mainly refer to \eqref{prob_batch_traj}. However, the results can be easily extended to \eqref{prob_batch_traj_nonordered} as well. The reason is that we mostly work with the term $\nabla_\theta \log q_i(\D^{i};\theta)$, and since $C_{\D^{i}}$ is not a function of $\theta$, for both cases we have
\begin{equation*}
\nabla_\theta \log q_i(\D^{i};\theta) = \sum_{\tau \in \D^i} \nabla_\theta \log q_i(\tau;\theta) = 	\sum_{\tau \in \D^i} \nabla_\theta \log \pi_i(\tau;\theta)
\end{equation*}
where the last equality is obtained using \eqref{traj_prob} along with the definition \eqref{pi_tau}.

\subsection{Lemmas}
\begin{lemma}\label{total_smoothness}
For any $i \in \{1,...,n\}$, let $f_i: \R^d \to W_i$ be a continuous function with $W_i \in \{\R, \R^d, \R^{1 \times d}, \R^{d \times d}\}$ such that $g(\theta)= f_n(\theta) ... f_1(\theta)$ is well defined. Furthermore, assume that for any $i$, the following holds:
\begin{enumerate}
\item $f_i$ is bounded, i.e., $\|f_i(\theta)\| \leq B_i$ for some nonnegative constant $B_i$ and any $\theta \in \R^d$.
\item $f_i$ is Lipschitz, i.e., $\|f_i(\theta) - f_i(\ttheta)\| \leq L_i \| \theta - \ttheta\|$ for some nonnegative constant $L_i$ and any $\theta, \ttheta \in \R^d$.	
\end{enumerate}
 Then, $g(\theta)$ is Lipschitz with parameter $L_g := \sum_{i=1}^n (L_i\prod_{j \neq i} B_j)$, i.e., for any $\theta$ and $\ttheta$, 
\begin{equation}
\| g(\theta) - g(\ttheta) \| \leq 	L_g \|\theta - \ttheta \|.
\end{equation}
\end{lemma}
\begin{proof}
We prove this result by induction on $n$. First, for $n=2$, note that
\begin{align}
\| g(\theta) - g(\ttheta) \| &= \left \| f_2(\theta) f_1(\theta) -  f_2(\ttheta) f_1(\ttheta) \right \|	\nonumber \\
& = \left \| f_2(\theta) f_1(\theta) -  f_2(\theta) f_1(\ttheta) + f_2(\theta) f_1(\ttheta) -  f_2(\ttheta) f_1(\ttheta) \right \| \nonumber \\
& \leq \left \| f_2(\theta) f_1(\theta) -  f_2(\theta) f_1(\ttheta) \right \| + \left \| f_2(\theta) f_1(\ttheta) -  f_2(\ttheta) f_1(\ttheta) \right \| \nonumber  \\
& \leq \| f_2(\theta) \| \| f_1(\theta) -  f_1(\ttheta)\| + \| f_1(\ttheta) \| \| f_2(\theta) - f_2(\ttheta)\| \nonumber \\
& \leq B_2 L_1 \|\theta - \ttheta \| + B_1 L_2 \|\theta - \ttheta \| = L_g \|\theta - \ttheta \| \label{n_2_smooth}
\end{align}
where the last inequality follows from the boundedness and Lipschitz property assumptions on $f_i$. 
Next, for $n \geq 3$, we assume the results holds for $n-1$, and we show it also holds for $n$. Note that if $f_{n}(\theta) ... f_1(\theta)$ is well defined, $f_m(\theta) ... f_1(\theta)$ is also well defined for any $m \leq n$, including $m=n-1$. Hence, by induction hypothesis
\begin{equation}\label{induction_hyp}
\| f_{n-1}(\theta) ... f_1(\theta) - f_{n-1}(\ttheta) ... f_1(\ttheta) \| \leq 	\tilde{L}_g \|\theta - \ttheta \|.
\end{equation}
where $\tilde{L}_g = \sum_{i=1}^{n-1} (L_i\prod_{j \neq i} B_j)$. Thus, $\tilde{g}(\theta) := f_{n-1}(\theta) ... f_1(\theta)$ is Lipschitz with parameter $\tilde{L}_g$. Also, it is bounded by $\prod_{j=1}^{n-1} B_j$. Finally, note that $\tilde{g}$ is a function from $\R^d$ to one of $\{\R, \R^d, \R^{1 \times d}, \R^{d \times d}\}$. Thus, using \eqref{n_2_smooth}, we obtain
\begin{align}
\| g(\theta) - g(\ttheta) \| = \left \| f_n(\theta) \tilde{g}(\theta) -  f_n(\ttheta) \tilde{g}(\ttheta) \right \| & \leq 	(B_n \tilde{L}_g + L_n \prod_{j=1}^{n-1} B_j) \|\theta - \ttheta \|.
\end{align}
However, it is easy to verify that in fact $B_n \tilde{L}_g + L_n \prod_{j=1}^{n-1} B_j = L_g$ and hence the proof is complete.
\end{proof}
\begin{lemma}\label{lemma_smoothness_expect}
For any $i \in \{1,...,n\}$, let $f_i: \R^d \to \R^m$ be a continuously differentiable function which is bounded by $B_f$, and is also Lipschitz with Lipschitz parameter $L_f$. Also, let $p(.;\theta)$ be a distribution on $\{f_i\}_{i=1}^n$ where probability of drawing $f_i$ is $p(i;\theta)$. We further assume there exists a non-negative constant $B_p$ such that for any $i$ and $\theta$
\begin{equation}
\|\nabla_\theta \log p(i;\theta)\| \leq B_p.	
\end{equation}
Then, the function $g(\theta) := \E_{p(i;\theta)} [f(i;\theta)] $ is Lipschitz with parameter $B_f B_p + L_f$.
\end{lemma}
\begin{proof}
First note that	
\begin{equation}\label{int_2_1}
\| \nabla_\theta p(i;\theta) \| 	= \|\nabla_\theta \log p(i;\theta)\|  p(i;\theta) \leq B_p ~ p(i;\theta).
\end{equation}
To show the result, it suffices to prove
\begin{equation}
\| \frac{\partial}{\partial \theta}	 g(\theta) \| \leq B_f B_p + L_f. 
\end{equation}
To show this, note that, by product rule, we have
\begin{align}
\frac{\partial}{\partial \theta}	 g(\theta) = \frac{\partial}{\partial \theta} (\sum_{i} f(i;\theta)p(i;\theta) ) = 	\sum_{i} p(i;\theta) \frac{\partial}{\partial \theta} f(i;\theta) + \sum_{i} \nabla p(i;\theta) f(i;\theta)^\top. 
\end{align}
As a result
\begin{align}
\| \frac{\partial}{\partial \theta}	 g(\theta)\| & \leq \sum_{i} p(i;\theta) \| \frac{\partial}{\partial \theta} f(i;\theta)\| + \sum_{i} \|\nabla p(i;\theta)\| \|f(i;\theta)\| \nonumber \\
& \leq L_f \sum_{i} p(i;\theta) + B_f B_p \sum_{i} p(i;\theta) \label{int_2_2} \\
& = L_f + B_f B_p \nonumber
\end{align}
where first part of \eqref{int_2_2} follows from the fact that $\| \frac{\partial}{\partial \theta} f(i;\theta)\| \leq L_f$ as $f(i;\theta)$ is Lipschitz with parameter $L_f$, and the second part of \eqref{int_2_2} is obtained using \eqref{int_2_1} along with boundedness assumption of $f_i$ functions.
\end{proof}
\section{Softmax Policy}\label{sec:softmax_example}
Consider the function $\phi:\A \times \MS \to \R^d$ as an arbitrary mapping from the space of actions-states to real-valued vectors with dimension $d$ which is the size of policy parameter $\theta$. Then, the softmax policy is given by\footnote{Through this example we suppress the task indices and mostly focus on softmax parametrization.}	
\begin{equation*}
\pi(a|s,\theta) = \frac{\exp (\phi(a,s)^\top \theta)}{\sum_{a'\in \A}\exp (\phi(a',s)^\top \theta)}.	
\end{equation*}
In this case, $\nabla_\theta \log 	\pi(a|s;\theta)$, which is known as the score function, admits the following characterization (see \citep{sutton2018reinforcement})
\begin{equation}\label{softmax_derivative}
\nabla_\theta \log \pi(a|s;\theta) = \phi(a,s) - \E_{a' \sim \pi(a'|s,\theta)}[\phi(a',s)].	
\end{equation}
Using this expression, we can show that the Hessian $\nabla_\theta^2 \log 	\pi(a|s;\theta)$ is equal to the negative of covariance matrix of random variable $\phi(a',s)$ when $a'$ is drawn from distribution $\pi(a'|s,\theta)$, i.e.,
\begin{align*}
& \nabla_\theta^2 \log \pi(a|s;\theta) \\ 
& = - \E_{a' \sim \pi(a'|s,\theta)} \left [ \left (\phi(a',s) - \E_{a'' \sim \pi(a''|s,\theta)}[\phi(a'',s)] \right )  \right. \\
& \left. \quad \quad \quad \quad \left (\phi(a',s) - \E_{a'' \sim \pi(a''|s,\theta)}[\phi(a'',s)] \right )^\top \right ].
\end{align*}
For more details regarding the derivation of $\nabla_\theta^2 \log \pi(a|s;\theta)$ please check Appendix \ref{app:Hessian_softmax}. 

According to the expressions for $\nabla_\theta \log \pi(a|s;\theta)$ and $\nabla_\theta^2 \log \pi(a|s;\theta)$, when we use a softmax policy, if we assume that the mapping norm $\|\phi(.,.)\|$ is bounded, then both conditions in Assumption \ref{assump_2} hold, i.e., $\| \nabla_\theta \log 	\pi(a|s;\theta) \|$ and $\| \nabla_\theta^2 \log 	\pi(a|s;\theta) \|$ would be both bounded for any action $a$, state $s$, and parameter $\theta$. Moreover, in Appendix \ref{app:Hessian_softmax}, we further show that the boundedness of $\|\phi(.,.)\|$ implies that the condition in Assumption \ref{assump_3} holds as well. 

Hence, at least for the softmax policy, the conditions in Assumptions~\ref{assump_2} and \ref{assump_3} hold, if the mapping $\phi$ has a bounded norm. {Note that in most applications, the mapping $\phi$ is a neural network and as the weights of neural networks are often bounded (or enforced to be bounded), $\|\phi(.,.)\|$ is uniformly upper bounded.}
\section{Multi-Step SG-MRL Method}\label{der_V_zeta}
We first start by characterizing $\nabla V_\zeta(\theta)$ for general $\zeta \geq 1$.
\begin{theorem}
Recall the definition of $V_\zeta(\theta)$ \eqref{Meta-RL_prob_zeta}. Then, its derivative can be expressed as
\begin{align}
& \nabla V_\zeta(\theta) = 	\E_{i \sim p} \E_{\{\D_{test,j}^{i}\}_{t=1}^{\zeta}} \left [ \prod_{t=1}^\zeta (I+\alpha \tnabla^2 J_i(\theta^{i,t-1}(\theta), \D_{test,t'}^i)) \nabla J_i(\theta^{i,\zeta}(\theta)) \right . \nonumber \\
& \left . + J_i \left ( \theta^{i,\zeta}(\theta) \right ) \sum_{t=1}^\zeta \left ( \prod_{t'=1}^{t-1} (I+\alpha \tnabla^2 J_i(\theta^{i,t'-1}(\theta), \D_{test,t'}^i)) \sum_{\tau \in \D_{test, t}^{i}} \nabla_\theta \log 	\pi_i(\tau;\theta^{i,t-1}(\theta)) \right ) \right ].
\end{align}	
\end{theorem}
\begin{proof}
To simplify the notation, let us define $\theta^{i,0}(\theta) := \theta$ and $\theta^{i,t}(\theta) := \Psi_i(...(\Psi_i(\theta, \D_{test,1}^{i})...),\D_{test, t}^{i})$ for $t \geq 1$. Then, $V_\zeta(\theta)$ can be cast as
\begin{equation}\label{Meta-RL_prob_zeta_2}
V_\zeta(\theta) = \E_{i \sim p} \left [ \E_{\{\D_{test,t}^{i}\}_{t=1}^{\zeta}} \left [ J_i \left ( \theta^{i,\zeta}(\theta) \right ) \right ] \right ].	
\end{equation}
Note that 
\begin{equation}\label{der_Psi}
\frac{\partial}{\partial \theta} \Psi_i(\theta, \D^i) = I + \alpha \tnabla^2 J_i(\theta, \D^i).	
\end{equation}
Now, using \eqref{der_Psi} along with chain rule, we have
\begin{align}\label{der_theta}
\frac{\partial}{\partial \theta}  \theta^{i,t}(\theta)  = & \frac{\partial}{\partial \theta} \left (  \Psi_i(...(\Psi_i(\theta, \D_{test,1}^{i})...),\D_{test, t}^{i}) \right ) = \prod_{t'=1}^t (I+\alpha \tnabla^2 J_i(\theta^{i,t'-1}(\theta), \D_{test,t'}^i))
\end{align}
for any $t \geq 1$.

Using the formulation for derivative of product of functions, we obtain: 
\begin{align}
 \nabla V_\zeta(\theta) &=  \nabla_\theta \E_{i \sim p} \left [  \sum_{\{\D_{test, t}^{i}\}_{t=1}^\zeta} J_i \left ( \theta^{i,\zeta}(\theta) \right ) \prod_{t=1}^\zeta q_i(\D_{test, t}^{i};\theta^{i,t-1}(\theta)) \right ] \nonumber \\
&= \E_{i \sim p} \left [ \sum_{\{\D_{test, t}^{i}\}_{t=1}^\zeta} \frac{\partial}{\partial \theta} \left ( J_i \left ( \theta^{i,\zeta}(\theta) \right ) \right ) \prod_{t=1}^\zeta q_i(\D_{test, t}^{i};\theta^{i,t-1}(\theta)) \right. \nonumber \\
& \left . + \sum_{\{\D_{test, t}^{i}\}_{t=1}^\zeta} \left ( J_i \left ( \theta^{i,\zeta}(\theta) \right ) \sum_{t=1}^\zeta \left ( \frac{\partial}{\partial \theta} \left ( q_i(\D_{test, t}^{i};\theta^{i,t-1}(\theta)) \right ) \prod_{\substack{t'=1 \\ t' \neq t}}^\zeta q_i(\D_{test, t'}^{i};\theta^{i,t'-1}(\theta)) \right )  \right )   \right]. \label{grad_V_3}
\end{align}
Now, note that, by using chain rule, we have
\begin{align}
\frac{\partial}{\partial \theta} \left ( q_i(\D_{test, t}^{i};\theta^{i,t-1}(\theta)) \right ) &= \frac{\partial}{\partial \theta}  \theta^{i,t-1}(\theta)  \nabla_\theta q_i(\D_{test, t}^{i};\theta^{i,t-1}(\theta)) \nonumber \\
&= 	\frac{\partial}{\partial \theta}  \theta^{i,t-1}(\theta) \nabla_\theta \log q_i(\D_{test, t}^{i};\theta^{i,t-1}(\theta))  q_i(\D_{test, t}^{i};\theta^{i,t-1}(\theta)) \label{grad_V_1}
\end{align}
Plugging \eqref{grad_V_1} in \eqref{grad_V_3}, we obtain
\begin{align}
\nabla & V_\zeta(\theta) = \nonumber \\
&= \E_{i \sim p} \left [ \sum_{\{\D_{test, t}^{i}\}_{t=1}^\zeta} \frac{\partial}{\partial \theta} \left ( J_i \left ( \theta^{i,\zeta}(\theta) \right ) \right ) \prod_{t=1}^\zeta q_i(\D_{test, t}^{i};\theta^{i,t-1}(\theta)) \right. \nonumber \\
& \left . + \sum_{\{\D_{test, t}^{i}\}_{t=1}^\zeta} \left ( J_i \left ( \theta^{i,\zeta}(\theta) \right ) \sum_{t=1}^\zeta \left ( \frac{\partial}{\partial \theta}  \theta^{i,t-1}(\theta) \nabla_\theta \log q_i(\D_{test, t}^{i};\theta^{i,t-1}(\theta)) \right ) \prod_{t=1}^\zeta q_i(\D_{test, t}^{i};\theta^{i,t-1}(\theta)) \right )   \right] \nonumber \\
& = \E_{i \sim p} \E_{\{\D_{test,j}^{i}\}_{t=1}^{\zeta}} \left [ \frac{\partial}{\partial \theta} \left ( J_i \left ( \theta^{i,\zeta}(\theta) \right ) \right ) + J_i \left ( \theta^{i,\zeta}(\theta) \right ) \sum_{t=1}^\zeta \left ( \frac{\partial}{\partial \theta}  \theta^{i,t-1}(\theta) \nabla_\theta \log q_i(\D_{test, t}^{i};\theta^{i,t-1}(\theta)) \right ) \right ] \nonumber \\
& = \E_{i \sim p} \E_{\{\D_{test,j}^{i}\}_{t=1}^{\zeta}} \left [ \frac{\partial}{\partial \theta} \theta^{i,\zeta}(\theta) \nabla J_i(\theta^{i,\zeta}(\theta)) \right . \nonumber \\
& \left . \quad \quad \quad + J_i \left ( \theta^{i,\zeta}(\theta) \right ) \sum_{t=1}^\zeta \left ( \frac{\partial}{\partial \theta}  \theta^{i,t-1}(\theta) \nabla_\theta \log q_i(\D_{test, t}^{i};\theta^{i,t-1}(\theta)) \right ) \right ] \label{grad_V_2}
\end{align}
where the last equality is derived by substituting $\frac{\partial}{\partial \theta} \left ( J_i \left ( \theta^{i,\zeta}(\theta) \right ) \right )$ by $\frac{\partial}{\partial \theta}  \theta^{i,t-1}(\theta) \nabla J_i(\theta^{i,\zeta}(\theta))$ by using chain rule. Now, we characterize $\nabla_\theta \log q_i(\D_{test, t}^{i};\theta^{i,t-1}(\theta))$ which appears in \eqref{grad_V_2}. First, recall that
\begin{equation*}
\nabla_\theta \log q_i(\D_{test, t}^{i};\theta^{i,t-1}(\theta)) = \sum_{\tau \in \D_{test, t}^{i}} \nabla_\theta \log q_i(\tau;\theta^{i,t-1}(\theta)).
\end{equation*}
Therefore, 
\begin{align}
\nabla_\theta \log q_i(\D_{test, t}^{i};\theta^{i,t-1}(\theta)) &= \sum_{\tau \in \D_{test, t}^{i}} \nabla_\theta \log q_i(\tau;\theta^{i,t-1}(\theta)) \nonumber \\
& = \sum_{\tau = ((s_j,a_j)_{j=0}^H) \in \D_{test, t}^{i}} \sum_{h=0}^H \nabla_\theta \log 	\pi_i(a_h|s_h;\theta^{i,t-1}(\theta)) \nonumber \\
& = \sum_{\tau \in \D_{test, t}^{i}} \nabla_\theta \log 	\pi_i(\tau;\theta^{i,t-1}(\theta)) \label{der_q_batch}
\end{align}
where the second equality follows from \eqref{traj_prob} and we used the notation \eqref{pi_tau} for the last equality. Plugging \eqref{der_q_batch} and \eqref{der_theta} in \eqref{grad_V_2}, we obtain
\begin{align}
& \nabla V_\zeta(\theta) = 	\E_{i \sim p} \E_{\{\D_{test,j}^{i}\}_{t=1}^{\zeta}} \left [ \prod_{t=1}^\zeta (I+\alpha \tnabla^2 J_i(\theta^{i,t-1}(\theta), \D_{test,t'}^i)) \nabla J_i(\theta^{i,\zeta}(\theta)) \right . \nonumber \\
& \left . + J_i \left ( \theta^{i,\zeta}(\theta) \right ) \sum_{t=1}^\zeta \left ( \prod_{t'=1}^{t-1} (I+\alpha \tnabla^2 J_i(\theta^{i,t'-1}(\theta), \D_{test,t'}^i)) \sum_{\tau \in \D_{test, t}^{i}} \nabla_\theta \log 	\pi_i(\tau;\theta^{i,t-1}(\theta)) \right ) \right ].
\end{align}
\end{proof}
\begin{algorithm}[tb]
\caption{Multi-Step SG-MRL}
\label{Algorithm3} 
\begin{algorithmic}
    \STATE {\bfseries Input:} Initial iterate $\theta_0$
	\REPEAT
    \STATE Draw a batch of \textit{i.i.d.} tasks (MDPs) $\B_k \subseteq \I$ from distribution $p$ and with size $B = |\B_k|$;
    \STATE Set $\theta_{k+1}^{i,0} = \theta_k $;
    \FOR{all $\T_i$ with $i \in \B_k$}
    \FOR{$t \gets 1 \text{ to } \zeta$}
    \STATE Sample a batch of trajectories  $\D_{in,t}^{i}$ w.r.t. $q_i(.;\theta_{k+1}^{i,t-1})$;
    \STATE Set $\theta_{k+1}^{i,t} = \theta_{k+1}^{i,t-1} + \al \tilde{\nabla}J_i(\theta_{k+1}^{i,t-1},\D_{in,t}^{i})$;
    \ENDFOR
    \ENDFOR
    \STATE Set $\theta_{k+1} = \displaystyle{\theta_k +  \beta \tnabla V_\zeta(\theta_k; \B_k, \{\D_{in,t}^i\}_{i,t}, \D_o^i) }$ where $\tnabla V_\zeta (.;.)$ is given by \eqref{unbiased_der_V_zeta}; 
    \STATE $k \gets k + 1$
    \UNTIL{not done}
\end{algorithmic}    
\end{algorithm}
As a consequence, 
\begin{align}\label{unbiased_der_V_zeta}
 & \tnabla V_\zeta(\theta; \B_k, \{\D_{in,t}^i\}_{i,t}, \D_o^i) := \frac{1}{B} \sum_{i \in \B_k} \left ( \prod_{t=1}^\zeta (I+\alpha \tnabla^2 J_i(\theta^{i,t-1}(\theta), \D_{in,t'}^i)) \tnabla J_i(\theta^{i,\zeta}(\theta), \D_{o}^i) \right . \nonumber \\
& \left . \quad + \tilde{J_i} \left ( \theta^{i,\zeta}(\theta), \D_{o}^i \right ) \sum_{t=1}^\zeta \left ( \prod_{t'=1}^{t-1} (I+\alpha \tnabla^2 J_i(\theta^{i,t'-1}(\theta), \D_{in,t'}^i)) \sum_{\tau \in \D_{in, t}^{i}} \nabla_\theta \log 	\pi_i(\tau;\theta^{i,t-1}(\theta)) \right ) \right )	
\end{align}
is an unbiased estimate of $\nabla V_\zeta(\theta)$ where $\B_k$ is a batch of tasks drawn independently from distribution $p$ and $\D_{in, t}^{i}$ and $\D_{o}^i$ are batch of trajectories drawn according to $q_i(.;\theta_{k+1}^{i,t-1})$ and $q_i(.;\theta_{k+1}^{i,\zeta})$, respectively. The steps of SG-MRL using this unbiased estimate are illustrated in Algorithm \ref{Algorithm3}.

\section{On Softmax Policy}\label{app:Hessian_softmax}
First, we show that
\begin{equation}\label{Hessian_policy_softmax}
\begin{split}	
\nabla_\theta^2 & \log \pi(a|s;\theta) =  \\ 
& - \E_{a' \sim \pi(a'|s,\theta)} \left [ \left (\phi(a',s) - \E_{a'' \sim \pi(a''|s,\theta)}[\phi(a'',s)] \right ) \left (\phi(a'',s) - \E_{a'' \sim \pi(a'|s,\theta)}[\phi(a'',s)] \right )^\top \right ].
\end{split}
\end{equation}
Note that
\begin{align}
\nabla_\theta^2 \log \pi(a|s;\theta)	 &= - \frac{\partial}{\partial \theta} \E_{a' \sim \pi(a'|s,\theta)}[\phi(a',s)] \nonumber \\
& = - \frac{\partial}{\partial \theta} \sum_{a' \in \A} \pi(a'|s,\theta)\phi(a',s) \nonumber \\
& = - \sum_{a' \in \A} \phi(a',s) \nabla_\theta \pi(a'|s,\theta)^\top \\
& = - \sum_{a' \in \A} \phi(a',s) \nabla_\theta \log \pi(a'|s,\theta)^\top \pi(a'|s,\theta) \label{softmax_Hessian_1} \\
& = - \E_{a' \sim \pi(a'|s,\theta)} \left [ \phi(a',s) \nabla_\theta \log \pi(a'|s,\theta)^\top \right ] \nonumber \\
& = - \E_{a' \sim \pi(a'|s,\theta)} \left [ \phi(a',s) \left (\phi(a',s) - \E_{a'' \sim \pi(a''|s,\theta)}[\phi(a'',s)] \right )^\top \right ] \label{softmax_Hessian_2} \\
& = - \E_{a' \sim \pi(a'|s,\theta)} \left [ \phi(a',s) \phi(a',s)^\top \right ] + \E_{a' \sim \pi(a'|s,\theta)}[\phi(a',s)] (\E_{a' \sim \pi(a'|s,\theta)}[\phi(a',s)])^\top \nonumber 
\end{align}
where \eqref{softmax_Hessian_1} follows from the log trick, i.e., the fact that $\nabla_\theta \pi(a'|s,\theta) = \nabla_\theta \log \pi(a'|s,\theta) \pi(a'|s,\theta) $, and \eqref{softmax_Hessian_2} is obtained using \eqref{softmax_derivative}.

\noindent Next, we assume $\phi(.,.)$ is bounded and want to show $\nabla_\theta^2 \log \pi(a|s;\theta)$ is a Lipschitz function of $\theta$. First, note that $\nabla_\theta \log \pi(a|s;\theta)$ given by \eqref{softmax_derivative} is bounded due to boundedness of $\phi(.,.)$. Thus, by Lemma \ref{lemma_smoothness_expect}, $\E_{a'' \sim \pi(a''|s,\theta)}[\phi(a',s)]$ is Lipschitz, and it is also bounded as $\phi(.,.)$ is bounded. Hence, the term 
\begin{equation*}
\left (\phi(a',s) - \E_{a'' \sim \pi(a''|s,\theta)}[\phi(a'',s)] \right ) \left (\phi(a'',s) - \E_{a'' \sim \pi(a'|s,\theta)}[\phi(a'',s)] \right )^\top	
\end{equation*}
is bounded, as it is also Lipschitz by Lemma \ref{total_smoothness}. Finally, applying Lemma \ref{lemma_smoothness_expect} one more time shows \eqref{Hessian_policy_softmax} is Lipschitz which completes the proof.
\section{Proof of Lemma \ref{lemma_all_in_all}}\label{proof_lemma_all_in_all}
\noindent \textbf{Proof of (1) \& (2)}: check \cite{shen2019hessian}.

\noindent \textbf{Proof of (3):} Note that it suffices to show the for one trajectory $\tau$, $u_i(\tau; \theta)$ is Lipschitz with parameter $\eta_\rho$ as
\begin{equation}
\| \tnabla^2 J_i (\theta_1, \D^i) - \tnabla^2 J_i (\theta_2, \D^i) \| \leq \frac{1}{|\D^i|} \sum_{\tau \in \D^i} \| u_i(\tau; \theta_1) - u_i(\tau; \theta_2)\|.	
\end{equation}
Let $\tau=(s_0,a_0,...,s_H,a_H)$. Recall that
\begin{align}
u_i(\tau; \theta) &= g_i(\tau;\theta)  \nabla_\theta \log q_i(\tau;\theta)^\top + \nabla_\theta^2 \nu_i(\tau;\theta) \nonumber \\
&= 	g_i(\tau;\theta) \left (\sum_{h=0}^H \nabla_\theta \log 	\pi_i(a_h|s_h;\theta) \right )^\top + \sum_{h=0}^H \nabla^2 \log \pi_i(a_h|s_h;\theta) \MR_i^h(\tau). \label{lemma_3_1}
\end{align}
We now show both terms in \eqref{lemma_3_1} are Lipschitz and characterize their Lipschitz parameters. First, note that $g_i(\tau;\theta)$ is bounded by $\eta_G$. Also, note that
\begin{align}
\|g_i(\tau;\theta_1) - g_i(\tau;\theta_2)\|	&= \| \sum_{h=0}^H \left ( \left (\nabla_\theta \log 	\pi_i(a_h|s_h;\theta_1) - \nabla_\theta \log 	\pi_i(a_h|s_h;\theta_2) \right ) \MR_i^h(\tau) \right ) \| \nonumber \\
& \leq \sum_{h=0}^H \left (  \| \nabla_\theta \log 	\pi_i(a_h|s_h;\theta_1) - \nabla_\theta \log 	\pi_i(a_h|s_h;\theta_2) \| \MR_i^h(\tau) \right ) \nonumber \\
& \leq \sum_{h=0}^H \left (  L \| \theta_1 - \theta_2 \| \MR_i^h(\tau) \right )\label{lemma_3_2} \\
& \leq L \| \theta_1 - \theta_2 \| \sum_{h=0}^H \frac{R \gamma^h}{1-\gamma} \label{lemma_3-3} \\
& \leq \frac{LR}{(1-\gamma)^2} \| \theta_1 - \theta_2 \| \nonumber
\end{align}
where \eqref{lemma_3_2} follows from Assumption \ref{assump_2} and \eqref{lemma_3-3} is obtained using the fact that $\MR_i^h(\tau) \leq \frac{R \gamma^h}{1-\gamma}$. In addition, $\sum_{h=0}^H \nabla_\theta \log 	\pi_i(a_h|s_h;\theta)$ is bounded by $(H+1) G$ and is Lipschitz with parameter $(H+1)L$ due to Assumption \ref{assump_2}. As a result, by Lemma \ref{total_smoothness}, the first term of \eqref{lemma_3_1}, i.e., $g_i(\tau;\theta) \left (\sum_{h=0}^H \nabla_\theta \log 	\pi_i(a_h|s_h;\theta) \right )^\top$ is Lipschitz with parameter $\eta_G (H+1)L + (H+1)G \frac{LR}{(1-\gamma)^2}$. Replacing $\eta_G$ implies that Lipschitz parameter is in fact ${2(H+1)GLR}/{(1-\gamma)^2}$.

For the second term of \eqref{lemma_3_1}, note that using Assumption \ref{assump_3} yields
\begin{align}
& \left \| \sum_{h=0}^H \left ( (\nabla^2 \log \pi_i(a_h|s_h;\theta) - \nabla^2 \log \pi_i(a_h|s_h;\theta)) \MR_i^h(\tau) \right ) \right \|	 \leq \sum_{h=0}^H \left (  \rho \| \theta_1 - \theta_2 \| \MR_i^h(\tau) \right )\nonumber \\
& \quad \quad \quad \leq  \rho \| \theta_1 - \theta_2 \| \sum_{h=0}^H \frac{R \gamma^h}{1-\gamma} \leq \frac{ \rho R}{(1-\gamma)^2} \| \theta_1 - \theta_2 \| \nonumber
\end{align}
where the second inequality once again follows from $\MR_i^h(\tau) \leq \frac{R \gamma^h}{1-\gamma}$. Adding up the Lipschitz parameters of both terms of \eqref{lemma_3_1} completes the proof.
\section{On Boundedness and Lipschitz Property of $\nabla V_\zeta(\theta)$}\label{app_smoothness}
In the following Theorem, we characterize boundedness and Lipschitz property of $\nabla V_\zeta(\theta)$ for any $\zeta \geq 1$. 
\begin{theorem}\label{thm_der_V_zeta_parameters}
Consider the objective function $V_\zeta$ defined in \eqref{Meta-RL_prob_zeta} for the case that $\alpha \in (0, {1}/{\eta_H}]$ where $\eta_H$ is given in Lemma \ref{lemma_all_in_all}. Suppose that the conditions in Assumptions~\ref{assump_1}-\ref{assump_3} are satisfied.  Then, for any $\theta \in \R^d$, the norm of $\nabla V_\zeta(\theta)$ is upper bounded by
\begin{equation}\label{bound_grad_V_zeta}
G_V(\zeta):= 2^\zeta (\eta_G + D_{in}GR(H+1))	= 2^\zeta GR \left (\frac{1}{(1-\gamma)^2}+D_{in}(H+1) \right ). 
\end{equation}
Moreover, $\nabla V_\zeta(\theta)$ is Lipschitz with parameter 
\begin{align}\label{smoothness_grad_V_zeta}
L_V(\zeta)& := \zeta 2^{\zeta-1} \alpha \eta_\rho \eta_G  + 2^{2\zeta} \eta_H \\
& + 2^{\zeta}  D_{in}(H+1) \left ( R \left ( 2^\zeta L + (\zeta+2^\zeta) D_{in}G^2(H+1)+(\zeta-1) \alpha \eta_\rho G \right ) + 2^{\zeta+1}\eta_G G \right ) \nonumber
\end{align}
where $\eta_G$ and $\eta_\rho$ are also defined in Lemma \ref{lemma_all_in_all}.
\end{theorem}
\begin{proof}
Recall from \eqref{grad_V_2} in Appendix \ref{der_V_zeta} that
\begin{align}
& \nabla V_\zeta(\theta) = 	\E_{i \sim p} \E_{\{\D_{test,j}^{i}\}_{t=1}^{\zeta}} \left [ \frac{\partial}{\partial \theta} \theta^{i,\zeta}(\theta) \nabla J_i(\theta^{i,\zeta}(\theta)) \right . \nonumber \\
& \left . \quad \quad \quad + J_i \left ( \theta^{i,\zeta}(\theta) \right ) \sum_{t=1}^\zeta \left ( \frac{\partial}{\partial \theta}  \theta^{i,t-1}(\theta) \nabla_\theta \log q_i(\D_{test, t}^{i};\theta^{i,t-1}(\theta)) \right ) \right ] \nonumber \\
& = \E_{i \sim p} \left [ \sum_{\{\D_{test,t}\}_{t=0}^\zeta} \left ( \prod_{t=1}^\zeta q_i(\D_{test, t}^{i};\theta^{i,t-1}(\theta)) \left ( \frac{\partial}{\partial \theta} \theta^{i,\zeta}(\theta) \nabla J_i(\theta^{i,\zeta}(\theta)) \right . \right . \right . \nonumber \\
& \left . \left . \left . \quad \quad \quad + J_i \left ( \theta^{i,\zeta}(\theta) \right ) \sum_{t=1}^\zeta \left ( \frac{\partial}{\partial \theta}  \theta^{i,t-1}(\theta) \nabla_\theta \log q_i(\D_{test, t}^{i};\theta^{i,t-1}(\theta)) \right ) \right ) \right ) \right ] \label{Lip_param_1}
\end{align}
where $\theta^{i,0}(\theta) := \theta$ and $\theta^{i,t}(\theta) := \Psi_i(...(\Psi_i(\theta, \D_{test,1}^{i})...),\D_{test, t}^{i})$ for $t \geq 1$. 
To show the desired result, we first characterize the boundedness and Lipschitz property of
\begin{align}\label{main_term_der_V}
\frac{\partial}{\partial \theta} \theta^{i,\zeta}(\theta) \nabla J_i(\theta^{i,\zeta}(\theta)) + 	J_i \left ( \theta^{i,\zeta}(\theta) \right ) \sum_{t=1}^\zeta \left ( \frac{\partial}{\partial \theta}  \theta^{i,t-1}(\theta) \nabla_\theta \log q_i(\D_{test, t}^{i};\theta^{i,t-1}(\theta)) \right )
\end{align}
for any $i$ and any sequence of batches $\{\D_{test,t}\}_{t=0}^\zeta$. In particular, we show \eqref{main_term_der_V} is bounded by $G_V(\zeta)$, and therefore, the bound holds for $\nabla V_\zeta(\theta)$ as well. Furthermore, we show a bound on the Lipschitz parameter of \eqref{main_term_der_V} which is independent of both $\{\D_{test,t}\}_{t=0}^\zeta$ and $i$, and we obtain it by showing each term in \eqref{main_term_der_V} is bounded and Lipschitz and then applying Lemma \ref{total_smoothness}. Finally, to show \eqref{smoothness_grad_V_zeta}, we use Lemma \ref{lemma_smoothness_expect}. 

We now start with studying boundedness and Lipschitz property of \eqref{main_term_der_V}. In this regard, first, we show the following lemma on the Lipschitz property of $\theta^{i,t}(\theta)$ and its derivative for any $t$:
\begin{lemma}\label{lemma_smoothness_theta}
Let $t \geq 1$, and recall that 	$\theta^{i,t}(\theta) := \Psi_i(...(\Psi_i(\theta, \D_{test,1}^{i})...),\D_{test, t}^{i})$ for a sequence of batch of trajectories $\{D_{test,j}^i\}_{j=1}^t$. Then, for any $\theta, \ttheta$, we have
\begin{enumerate}
	\item 
\begin{equation}
\| \frac{\partial}{\partial \theta}  \theta^{i,t}(\theta)\| \leq (1+\alpha \eta_H)^{t} , \quad \text{ and thus } \| \theta^{i,t}(\theta) - \theta^{i,t}(\ttheta)\| \leq (1+\alpha \eta_H)^t \|\theta - \ttheta \|,
\end{equation}
	\item 
\begin{equation}
\| \frac{\partial}{\partial \theta}  \theta^{i,t}(\theta) - \frac{\partial}{\partial \theta}  \theta^{i,t}(\ttheta)\| \leq t \alpha \eta_\rho (1+\alpha \eta_H)^{t-1} \|\theta - \ttheta \|	
\end{equation}
\end{enumerate}
where $\eta_H$ and $\eta_\rho$ are given in Lemma \ref{lemma_all_in_all}. 
\end{lemma}
\begin{proof}
Recall from \eqref{der_theta} in Appendix \ref{der_V_zeta} that
\begin{align}
\frac{\partial}{\partial \theta}  \theta^{i,t}(\theta)  = \prod_{t'=1}^t (I+\alpha \tnabla^2 J_i(\theta^{i,t'-1}(\theta), \D_{test,t'}^i))
\end{align}	
In part (2) of Lemma \ref{lemma_all_in_all} we showed that for any $t'$, $\|\tnabla^2 J_i(\theta^{i,t'-1}(\theta), \D_{test,t'}^i)\| \leq \eta_H$, and this immediately implies the first result. 

\noindent Also, for the second result, note that for each $t'$, $I+\alpha \tnabla^2 J_i(\theta^{i,t'-1}(\theta), \D_{test,t'}^i)$ is bounded by $1+\alpha \eta_H$ due to part (2) of Lemma \ref{lemma_all_in_all}, and is Lipschitz with parameter $\alpha \eta_\rho$ by part (3) of Lemma \ref{lemma_all_in_all}. Thus, using Lemma \ref{total_smoothness} gives us the desired result. 
\end{proof}
Next, we go step by step and study the boundedness and Lipschitz property of each term in \eqref{main_term_der_V}. Throughout this process, we also use the assumption $\alpha \leq 1/\eta_H$ to replace the term $(1+\alpha \eta_H)$ by $2$ and simplify the results.
\begin{enumerate}[label=(\roman*)]
\item As we showed in Lemma \ref{lemma_smoothness_theta}, $\frac{\partial}{\partial \theta}  \theta^{i,\zeta}(\theta)$ is bounded by $2^{\zeta}$ and also Lipschitz with parameter $\zeta \alpha \eta_\rho 2^{\zeta-1}$. 
Also, $\nabla J_i(\theta^{i,\zeta}(\theta))$ is bounded by $\eta_G$ by part (1) of Lemma \ref{lemma_all_in_all} and is Lipschitz with parameter $\eta_H 2^{\zeta}$ by using part (2) of Lemma \ref{lemma_all_in_all} and Lemma \ref{lemma_smoothness_theta} along with the fact that the Lipschitz parameter of combination of functions is the product of their Lipschitz parameters.
Thus, using Lemma \ref{total_smoothness}, the term $\frac{\partial}{\partial \theta} \theta^{i,\zeta}(\theta) \nabla J_i(\theta^{i,\zeta}(\theta))$ in total is bounded by $\eta_G2^{\zeta} $ and is Lipschitz with parameter $\zeta 2^{\zeta-1} \alpha \eta_\rho \eta_G  + 2^{2\zeta} \eta_H $.

\item For any $t$, and by Lemma \ref{lemma_smoothness_theta}, $\frac{\partial}{\partial \theta}  \theta^{i,t-1}(\theta)$ is bounded by $2^{t-1}$ and its Lipschitz parameter is bounded by $(t-1)2^{t-1}\alpha \eta_\rho$. 

\noindent Also, it is easy to check 
\begin{equation}
\| \nabla_\theta \log q_i(\D_{test, t}^{i}; \theta) \| \leq D_{in}G(H+1), \quad \| \nabla_\theta^2 \log q_i(\D_{test, t}^{i}; \theta) \| \leq D_{in}L(H+1).	
\end{equation}
Hence, $\nabla_\theta \log q_i(\D_{test, t}^{i};\theta^{i,t-1}(\theta))$ is bounded by $D_{in}G(H+1)$. In addition, since $\theta^{i,t-1}(\theta)$ is Lipschitz with parameter $2^{t-1}$, the whole $\nabla_\theta \log q_i(\D_{test, t}^{i};\theta^{i,t-1}(\theta))$ is Lipschitz with parameter $2^{t-1} D_{in}L(H+1)$.

\noindent Thus, for any $t$, the term $\frac{\partial}{\partial \theta}  \theta^{i,t-1}(\theta) \nabla_\theta \log q_i(\D_{test, t}^{i};\theta^{i,t-1}(\theta))$ is bounded by $2^{t-1}D_{in}G(H+1)$ and is Lipschitz with parameter $D_{in}(H+1)(2^{2t-2}L + (t-1)2^{t-1}\alpha \eta_\rho G)$. As a consequence, the sum 
\begin{equation}
\sum_{t=1}^\zeta \left ( \frac{\partial}{\partial \theta}  \theta^{i,t-1}(\theta) \nabla_\theta \log q_i(\D_{test, t}^{i};\theta^{i,t-1}(\theta)) \right )	
\end{equation}
is bounded by $2^{\zeta}D_{in}G(H+1)$ and its Lipschitz parameter is bounded by 
$$D_{in}(H+1)\left ( 4^\zeta L + 2^\zeta (\zeta-1) \alpha \eta_\rho G \right ).$$

\item $J_i \left ( \theta^{i,\zeta}(\theta) \right )$ is clearly bounded by $R$. Also, by part(1) of Lemma \ref{lemma_all_in_all} $J_i$ is Lipschitz with parameter $\eta_G$ and also by Lemma \ref{lemma_smoothness_theta}, $\theta^{i,\zeta}(\theta)$ is Lipschitz with parameter $2^\zeta$. Using these two along with the fact that Lipschitz parameter of combination of functions is equal to the product of their Lipschitz parameters, implies that $J_i \left ( \theta^{i,\zeta}(\theta) \right )$ is Lipschitz with parameter $2^\zeta \eta_G$.

\item Therefore, using (iv) and (v), the whole term 
\begin{equation}
\prod_{t=1}^\zeta q_i(\D_{test, t}^{i};\theta^{i,t-1}(\theta)) J_i \left ( \theta^{i,\zeta}(\theta) \right ) \sum_{t=1}^\zeta \left ( \frac{\partial}{\partial \theta}  \theta^{i,t-1}(\theta) \nabla_\theta \log q_i(\D_{test, t}^{i};\theta^{i,t-1}(\theta)) \right )	
\end{equation}
is bounded by $2^{\zeta}D_{in}GR(H+1)$ and, by Lemma \ref{total_smoothness}, its Lipschitz parameter is bounded by 
\begin{equation*}
D_{in}R(H+1)\left ( 4^\zeta L + 2^\zeta (\zeta-1) \alpha \eta_\rho G \right ) + 2^{2\zeta}D_{in}G(H+1) \eta_G + R \zeta 2^{\zeta}D_{in}^2G^2(H+1)^2.
\end{equation*}
which can be simplified and written as
\begin{equation*}
2^{\zeta}  D_{in}(H+1) \left ( R \left ( 2^\zeta L + \zeta D_{in}G^2(H+1)+(\zeta-1) \alpha \eta_\rho G \right ) + 2^\zeta\eta_G G \right )	
\end{equation*}
\end{enumerate}
Part (i) and (iv) together imply that \eqref{main_term_der_V} is bounded by 
\begin{equation}
2^\zeta (\eta_G + D_{in}GR(H+1))	= 2^\zeta GR \left (\frac{1}{(1-\gamma)^2}+D_{in}(H+1) \right ) 
\end{equation}
which is in fact $G_V(\zeta)$. Since this upper bound is independent of $i$ and $\{D_{test,t}^i\}_t$, it also holds for $\nabla V_\zeta(\theta)$, and this completes the proof of \eqref{bound_grad_V_zeta}.

\noindent Also, part (i) and (iv) together imply that \eqref{main_term_der_V} is Lipschitz with parameter 
\begin{equation}\label{Lip_par_3}
\zeta 2^{\zeta-1} \alpha \eta_\rho \eta_G  + 2^{2\zeta} \eta_H  + 2^{\zeta}  D_{in}(H+1) \left ( R \left ( 2^\zeta L + \zeta D_{in}G^2(H+1)+(\zeta-1) \alpha \eta_\rho G \right ) + 2^\zeta\eta_G G \right ).
\end{equation}

Now, to derive the Lipschitz parameter of $\nabla V_\zeta(\theta)$ itself, we use Lemma \ref{lemma_smoothness_expect}. To do so, first we show the following lemma.
\begin{lemma}\label{lemma_smoothness_q_pi} 
Recall definition of $q_i(\D^i;\theta)$ \eqref{prob_batch_traj} for some MDP $\M_i$, batch of trajectories $\D^i$ and policy parameter $\theta \in \R^d$. Then, for any $\D^i$ and $\theta$, we have
\begin{equation}
\| \nabla_\theta \log q_i(\D^i;\theta)\| \leq |\D^i| (H+1)G.
\end{equation}
\end{lemma}
\begin{proof}
Note that
\begin{align}
\| \nabla_\theta \log q_i(\D^i;\theta) \| &= \left \| \sum_{\tau \in \D^{i}} \nabla_\theta \log 	\pi_i(\tau;\theta) \right \| \label{lemma_4_1} \\
& \leq  |D^i| \max_{\tau=(s_0,a_0,...,s_H,a_H)} \| \nabla_\theta \log 	\pi_i(\tau;\theta) \|   \nonumber \\
& \leq |D^i| \max_{\tau=(s_0,a_0,...,s_H,a_H)} \sum_{h=0}^H \| \nabla_\theta \log 	\pi_i(a_h|s_h;\theta) \|  \label{lemma_4_2} \\
& \leq |\D^i| (H+1)G
\end{align}
where \eqref{lemma_4_1} follows from \eqref{prob_batch_traj} and \eqref{lemma_4_2} is obtained using \eqref{pi_tau} along with Assumption \ref{assump_2}.
\end{proof}
Using this lemma, we have
\begin{align}
\| \nabla_\theta \left ( \log \prod_{t=1}^\zeta q_i(\D_{test, t}^{i};\theta^{i,t-1}(\theta)) \right ) \| & \leq  \sum_{t=1}^\zeta \| \frac{\partial}{\partial \theta}  \theta^{i,t-1}(\theta) \nabla_\theta \log q_i(\D_{test, t}^{i};\theta^{i,t-1}(\theta)) \| \nonumber \\
& \leq  |D_{in}| (H+1)G 	\sum_{t=1}^\zeta \| \frac{\partial}{\partial \theta}  \theta^{i,t-1}(\theta) \| \label{lip_par_4} \\
& \leq |D_{in}| (H+1)G 	\sum_{t=1}^\zeta 2^{t-1} \label{lip_par_5} \\
& \leq 2^\zeta |D_{in}| (H+1)G \nonumber
\end{align}
where \eqref{lip_par_4} follows from Lemma \ref{lemma_smoothness_q_pi} and \eqref{lip_par_5} is obtained using Lemma \ref{lemma_smoothness_theta}. Now, using this bound and \eqref{Lip_par_3} along with Lemma \ref{lemma_smoothness_expect} implies that $\nabla V_\zeta(\theta)$ is Lipschitz with parameter 
\begin{equation}\label{Lip_par_6}
\zeta 2^{\zeta-1} \alpha \eta_\rho \eta_G  + 2^{2\zeta} \eta_H  + 2^{\zeta}  D_{in}(H+1) \left ( R \left ( 2^\zeta L + (\zeta+2^\zeta) D_{in}G^2(H+1)+(\zeta-1) \alpha \eta_\rho G \right ) + 2^{\zeta+1}\eta_G G \right )
\end{equation}
which completes the proof of \eqref{smoothness_grad_V_zeta}.
\end{proof}

In particular, for $\zeta=1$, it is easy to verify the Lipschitz parameter of $\nabla V_1 (\theta)$ admits the upper bound
\begin{equation}
\alpha \eta_\rho \eta_G + 4 \eta_H + 8 R D_{in}(H+1) (L + D_{in}G^2(H+1)).
\end{equation}
Finally, we state the following result on boundedness of unbiased estimate of $\nabla V_\zeta(\theta)$ used in update of MAML (Algorithm \ref{Algorithm3}).
\begin{lemma}\label{lemma_bound_gradient}
Recall $\tnabla V_\zeta(\theta_k; \B_k, \{\D_{in,t}^i\}_{i,t}, \D_o^i)$	\eqref{unbiased_der_V_zeta} in Multi-step MAML algorithm (Algorithm \ref{Algorithm3}) for the case that $\alpha \in (0, {1}/{\eta_H}]$ where $\eta_H$ is given in Lemma \ref{lemma_all_in_all}. Suppose that the conditions in Assumptions~\ref{assump_1}-\ref{assump_3} are satisfied. Then, at iteration $k+1$, and for any choice of $\B_k$, $\{\D_o^i\}_i$ and $\{\D_{in,t}^i\}_{i,t}$, we have
\begin{equation}
\| \tnabla V_\zeta(\theta_k; \B_k, \{\D_{in,t}^i\}_{i,t}, \D_o^i) \| \leq G_V(\zeta)	
\end{equation}
where $G_V(\zeta)$ is given in Theorem \ref{thm_der_V_zeta_parameters}.
\end{lemma}
\begin{proof}
We skip the details of the proof as it can be done very similar to how we proved \eqref{main_term_der_V} in Theorem \ref{thm_der_V_zeta_parameters}. In particular, note that for any choice of $\D_{o}^i$ 
\begin{equation}
\|\tnabla J_i(\theta^{i,\zeta}(\theta), \D_{o}^i)\| \leq \eta_G, 	\quad \| \tilde{J}_i(\theta^{i,\zeta}(\theta), \D_{o}^i) \| \leq R
\end{equation}
where the first one follows from Lemma \ref{lemma_all_in_all} and the second one is an immediate result of Assumption \ref{assump_1}.
\end{proof}
\section{Proof of Theorem \ref{Thm_Main_Result_1}}\label{proof_Thm_Main_Result}
We first state the general statement of the theorem for any $\zeta \geq 1$.
\begin{theorem}\label{Thm_Main_Result_zeta} 
Consider the objective function $V_\zeta$ defined in \eqref{Meta-RL_prob_zeta} for the case that $\alpha \in (0, {1}/{\eta_H}]$ where $\eta_H$ is given in Lemma \ref{lemma_all_in_all}. Suppose that the conditions in Assumptions~\ref{assump_1}-\ref{assump_3} are satisfied, and recall the definitions $L_V(\zeta)$ and $G_V(\zeta)$ from Theorem \ref{thm_der_V_zeta_parameters}. Consider running Multi-step SG-MRL (Algorithm \ref{Algorithm3}) with $\beta \in (0,1/L_V(\zeta)]$.
Then, for any $1>\eps >0$, MAML finds a solution $\theta_\eps$ such that
\begin{equation}
 \E[ \| \nabla V_\zeta(\theta_\eps) \|^2] \leq  \frac{2 G_V(\zeta)^2 L_V(\zeta) \beta}{BD_o}  + \eps^2
\end{equation}
after at most running for 
\begin{equation}
\bigO(1) \frac{R}{\beta} \min \left \{ \frac{1}{\eps^2},
\frac{B D_o}{G_V(\zeta)^2 L_V(\zeta) \beta} \right \} 
\end{equation}
iterations.
\end{theorem}
\begin{proof}
Throughout the proof, we use $G_V$ and $L_V$ instead of 	$G_V(\zeta)$ and $L_V(\zeta)$, respectively, to simplify the notation. Also, we denote the filtration till the end of iteration $k$ by $\F_k$. 

As we previously discussed, $\tnabla V_\zeta(\theta_k; \B_k, \{\D_{in,t}^i\}_{i,t}, \D_o^i)$ is an unbiased estimate of $\nabla V_\zeta(\theta_k)$ at iteration $k+1$. In the following lemma, we upper bound the variance of this estimation.
\begin{lemma}\label{bound_variance}
Recall the definition of $\tnabla V_\zeta(\theta_k; \B_k, \{\D_{in,t}^i\}_{i,t}, \D_o^i)$	 \eqref{unbiased_der_V_zeta} in Multi-step SG-MRL algorithm (Algorithm \ref{Algorithm3}) for the case that $\alpha \in (0, {1}/{\eta_H}]$ where $\eta_H$ is given in Lemma \ref{lemma_all_in_all}. Suppose that the conditions in Assumptions~\ref{assump_1}-\ref{assump_3} are satisfied. Then, at iteration $k+1$, and for any choice of $\B_k$, $\{\D_o^i\}_i$ and $\{\D_{in,t}^i\}_{i,t}$, we have
\begin{equation}
\E \left [ \left \| \tnabla V_\zeta(\theta_k; \B_k, \{\D_{in,t}^i\}_{i,t}, \D_o^i) - \nabla V_\zeta (\theta_k) \right \|^2 \right ] \leq \frac{G_V^2}{BD_o}	
\end{equation}
where $G_V$ is given in Theorem \ref{thm_der_V_zeta_parameters}.
\end{lemma}
\begin{proof}
Note that
\begin{equation}
\tnabla V_\zeta(\theta_k; \B_k, \{\D_{in,t}^i\}_{i,t}, \D_o^i) = \frac{1}{BD_o}\sum_{i \in \B_k} \sum_{\tau \in \D_o^i} \tnabla V_\zeta(\theta_k; \{i\}, \{\D_{in,t}^i\}_{i,t}, \{\tau\}),
\end{equation}
where for any $i$ and $\tau \in \D_o^i$, $\tnabla V_\zeta(\theta_k; \{i\}, \{\D_{in,t}^i\}_{i,t}, \{\tau\})$ is an unbiased estimate of $\nabla V_\zeta (\theta_k)$, and by Lemma \ref{lemma_bound_gradient}, its second moment is bounded by $G_V^2$. Also, note that $\tnabla V_\zeta(\theta_k; \{i\}, \{\D_{in,t}^i\}_{i,t}, \{\tau\})$ are independent for different $i$ and $\tau$. Finally, to complete the proof, we use the well-known fact that if $\{X_i\}_{i=1}^n$ are independent with mean $\mu$, and for each $i$, variance  of $X_i$ is upper bounded by $\sigma^2$, then
\begin{equation*}
\E \left [ \left \| \frac{X_1+...+X_n}{n} - \mu \right \|^2 \right ]	\leq \frac{\sigma^2}{n}.
\end{equation*}
\end{proof}
Now, we get back to the proof of the main result. From now, and to simplify the notation, we use $\tnabla V_\zeta(\theta_k)$ to denote $\tnabla V_\zeta(\theta_k; \B_k, \{\D_{in,t}^i\}_{i,t}, \D_o^i)$. Next, note that, using the smoothness property of $\nabla V_\zeta(\theta)$, we have \cite{nesterov_convex}
\begin{equation}\label{smooth_nest}
\left \vert V_\zeta(\theta_{k+1}) - V_\zeta(\theta_k) - \nabla V_\zeta(\theta_k)^\top (\theta_{k+1} - \theta_k) \right \vert	\leq \frac{L_V^2}{2} \|\theta_{k+1} - \theta_k \|^2.
\end{equation}
Recall that, at iteration $k+1$, MAML performs
\begin{equation}\label{main_proof_1}
\theta_{k+1} = \theta_k + \beta 	\tnabla V_\zeta(\theta_k).
\end{equation}
Plugging this in \eqref{smooth_nest}, we obtain
\begin{align}
- V_\zeta(\theta_{k+1}) & \leq - V_\zeta(\theta_k) - \nabla V_\zeta(\theta_k)^\top (\theta_{k+1} - \theta_k) + \frac{L_V^2}{2} \|\theta_{k+1} - \theta_k \|^2 \nonumber \\
& =   - V_\zeta(\theta_k) - \beta  \nabla V_\zeta(\theta_k)^\top \tnabla V_\zeta(\theta_k) + \frac{L_V^2}{2} \beta^2 \|\tnabla V_\zeta(\theta_k)\|^2 
\end{align}
where the last equality follows from \eqref{main_proof_1}. Next, taking expectation from both sides and conditioning on $\F_k$, implies
\begin{align}
- \E [ & V_\zeta(\theta_{k+1}) | \F_k] \nonumber \\
 &  \leq - V_\zeta(\theta_k) - \beta  \|\nabla V_\zeta(\theta_k)\|^2 + \frac{L_V}{2} \beta^2 \left (\|\nabla V_\zeta(\theta_k)\|^2 + \E \left [\| \tnabla V_\zeta(\theta_k) -  \nabla V_\zeta(\theta_k)\|^2 |\F_k \right ] \right ) \label{main_proof_2} \\
& \leq - V_\zeta(\theta_k) - \frac{\beta}{2} \|\nabla V_\zeta(\theta_k)\|^2 + \frac{G_V^2 L_V \beta^2}{2 BD_o} \label{main_proof_3}
\end{align}
 where the first inequality is obtained using the fact that $\tnabla V_\zeta(\theta_k)$ is an unbiased estimate of $\nabla V_\zeta(\theta_k)$ and $\nabla V_\zeta(\theta_k)$ is deterministic condition on $\F_k$. \eqref{main_proof_3} is also an immediate result of Lemma \ref{bound_variance} along with $\beta \leq 1/L_V$.
 
Taking another expectation from both sided of \eqref{main_proof_3}, and using tower rule, we obtain
 \begin{equation}\label{main_proof_4}
- \E [ V_\zeta(\theta_{k+1}) ]	\leq - \E[V_\zeta(\theta_k)] - \frac{\beta}{2} \E \left[\|\nabla V_\zeta(\theta_k)\|^2 \right ] + \frac{G_V^2 L_V \beta^2}{2 BD_o} .
 \end{equation}
We complete the proof by contradiction. Assume, the desired result does not hold for the first $T$ iterations, i.e., 
\begin{equation}
 \E[ \| \nabla V_\zeta(\theta_k) \|^2] \geq  \frac{2G_V^2 L_V \beta}{BD_o} + \eps^2	
\end{equation}
for any $0 \leq k \leq T-1$. Then, by \eqref{main_proof_4}, for any $0 \leq k \leq T-1$, we have
 \begin{equation}\label{main_proof_5}
 -\E [ V_\zeta(\theta_{k+1}) ]	\leq - \E[V_\zeta(\theta_k)] - \frac{\beta \eps^2}{2} - \frac{G_V^2 L_V \beta^2}{2 BD_o}	.
 \end{equation}
 Adding up this result for $k=0,...,T-1$ yields
 \begin{equation}\label{main_proof_6}
 - \E [ V_\zeta(\theta_{T}) ]	\leq - \E[V_\zeta(\theta_0)] - T \left (\frac{\beta \eps^2}{2} + \frac{G_V^2 L_V \beta^2}{2 B D_o} \right )	.
 \end{equation}
Note that, by Assumption \ref{assump_1}, both $\E [ V_\zeta(\theta_{T}) ] $ and $\E[V_\zeta(\theta_0)] $ have values between zero and $R$, and thus, their difference is bounded by $R$. Therefore,
\begin{equation}
T \left (\frac{\beta \eps^2}{2} + \frac{G_V^2 L_V \beta^2}{2 BD_o} \right ) \leq R	
\end{equation}
which gives us the desired result.
\end{proof}
\section{More Details on the Numerical Experiment Section}
\label{appendix_experiments}

In this section of the Appendix we detail our experimental setup beyond the description given in Section~\ref{sec:exp}. We use a neural network policy 
with two $100$-unit hidden layers and ReLU activations. For simplicity, we use vanilla policy gradient (VPG) for both the inner adaption steps and the outer meta steps.  

In all cases, we train both algorithms for $500$ (meta-)epochs, using a meta-batch size of $20$ tasks for $2$D-navigation and $40$ tasks for the locomotion one. For all tasks, we use $20$ episodes per adaptation step.
All rewards are discounted with a factor $\gamma=0.99$.
We use a horizon $H=100$ for $2$D-navigation and $H=200$ for locomotion tasks.
Next, we use a learning rate of $0.1$ for the inner steps, and $0.001$ for the outer ones. Finally, all experiments are averaged over $10$ random seeds.

The MuJoCo locomotion environments we consider are
\begin{itemize}
    \item \textbf{Half-Cheetah Random Direction} 
which simulates the dynamics of a ``cheetah" robot which is trained to move fast.
In this environment, each task is a goal direction (forward/backward) and the reward at each timestep is given by the magnitude of the agent's velocity.
    \item \textbf{Half-Cheetah Random Velocity} which uses the same ``cheetah" robot, but now each task is a goal velocity. The reward at each timestep is given by the negative of the absolute difference between the current and goal velocities. 
    \item \textbf{Swimmer Random Velocity}
which simulates the dynamics of a planar ``swimmer" robot in a viscous liquid. The swimmer needs to use viscous drag to propel itself. Like with the other direction environment, each task is a goal direction (forward/backward) and the reward at each timestep is given by the magnitude of the agent's velocity.
\end{itemize}

For each of the environments, we present results using $1, 2$, and $3$ gradient steps.

Finally, we use MuJoCo~\cite{mujoco} license and perform all experiments on an internal server using $2$ NVIDIA V100 GPUs.

\end{document}